\documentclass{article} 
\usepackage{iclr2027_conference,times}

\usepackage{amsmath,amsfonts,bm}

\def\eqref#1{equation~\ref{#1}}
\def\1{\bm{1}}

\DeclareMathAlphabet{\mathsfit}{\encodingdefault}{\sfdefault}{m}{sl}
\SetMathAlphabet{\mathsfit}{bold}{\encodingdefault}{\sfdefault}{bx}{n}

\usepackage{hyperref}
\usepackage{url}

\usepackage[utf8]{inputenc} 
\usepackage[T1]{fontenc}    
\usepackage{hyperref}       
\usepackage{url}            
\usepackage{booktabs}       
\usepackage{amsfonts}       
\usepackage{nicefrac}       
\usepackage{wrapfig}        
\usepackage{microtype}      
\usepackage{xcolor}         
\usepackage{graphicx}
\usepackage{amsthm}
\usepackage{bm}
\usepackage{amsmath}
\usepackage{amssymb}
\usepackage{mathtools}
\usepackage{microtype}
\usepackage{graphicx}
\usepackage{subcaption}
\usepackage{booktabs} 
\usepackage{adjustbox}
\usepackage{pifont}
\newtheorem{theorem}{Theorem}
\newtheorem{lemma}{Lemma}
\usepackage{multirow}

\usepackage{enumitem}

\title{Gaussian Relational Graph Transformer}
\iclrfinalcopy

\author{
	\textbf{Zezhong Ding$^{1, 3, \dag}$, \quad Jin Li$^{2, 3, \dag}$,\quad Xugang Wang$^4$,\quad Xike Xie$^{2, 3}$\thanks{Corresponding Author \quad  $^\dag$Equal Contribution}}\vspace{3pt} \\
     {\small $^1$School of Artificial Intelligence and Data Science, University of Science and Technology of China (USTC)} \\
	 {\small $^2$School of Biomedical Engineering, USTC} \\
	{\small $^3$Data Darkness Lab, Suzhou Institute for Advanced Research, USTC} \\
    {\small $^4$Chinese Academy of Sciences} \\ 
	\texttt{\small \{zezhongding,lijinstu\}@mail.ustc.edu.cn, xgwang@iaii.ac.cn, xkxie@ustc.edu.cn} 
}

\begin{document}

\maketitle
\pagestyle{plain}
\thispagestyle{plain}
\begin{abstract}
Relational graph learning enables predictive modeling over relational databases by directly capturing dependencies across interconnected tables. While relational graph Transformers extend the receptive field beyond local message passing, a larger receptive field does not necessarily provide more useful information: sampling must preserve relational structure without introducing excessive semantically irrelevant nodes, while attention must account for the temporal relevance of the retained information. We propose GelGT, a Gaussian relational graph transformer that explicitly addresses these challenges. 
GelGT introduces a structure-semantic collaborative sampling strategy to preserve structural connectivity while filtering irrelevant semantic information, and incorporates a Gaussian graph attention mechanism with a learnable Gaussian bias on the sampled subgraphs to dynamically encode temporal dependencies. 
We provide theoretical analysis characterizing the structural preservation, semantic refinement, and temporal discrimination of these mechanisms.
Experiments on \textbf{7} real-world relational datasets covering \textbf{21} prediction tasks show that GelGT consistently outperforms existing relational graph learning methods, with improvements of up to \textbf{13.8\%}. Our source code is available at \url{https://github.com/USTC-DataDarknessLab/GelGT}.
\end{abstract}

\section{Introduction}
\label{sec:intro}

Relational databases are widely used to organize interconnected entities across multiple tables in applications such as
e-commerce, finance, and healthcare.  
Predictive modeling over such databases traditionally relies on tabular models \citep{ChenG16, Shwartz-ZivA22, grinsztajn2022why, chen2023trompt, voskou2024transformers}, which typically require manually joining tables and engineering features for individual tasks \citep{DwivediKHL25, KanterV15, LamBMMWKBSWS21}.
Recent {\it relational graph learning} approaches instead represent rows as nodes and primary–foreign key relations as edges \citep{0001RHHHDFLYZHL24, ChenKL25, FeyHHLR0YYL24}, enabling models to learn directly from both entity attributes and cross-table dependencies \citep{abs-2002-02046, FeyHHLR0YYL24}.

Existing relational graph learning methods (e.g., RDL \citep{0001RHHHDFLYZHL24} and RelGNN \citep{ChenKL25}) are primarily built upon message-passing. 
While effective for local dependencies, message passing becomes increasingly restrictive when useful information lies several hops away due to information decay during propagation \citep{WuJWMGS21, WuZLWY22}. This limitation is particularly relevant to relational databases, where predictions can require information distributed across multiple tables.
More recently, RelGT \citep{dwivedi2026relational} applies Transformer attention over sampled relational subgraphs (Figure~\ref{fig:intro}). This enlarges the receptive field, 
making it increasingly important to determine which information should be retained and how it should contribute to prediction. 

\begin{figure*}[t]
    \centering
    \includegraphics[width=\textwidth]{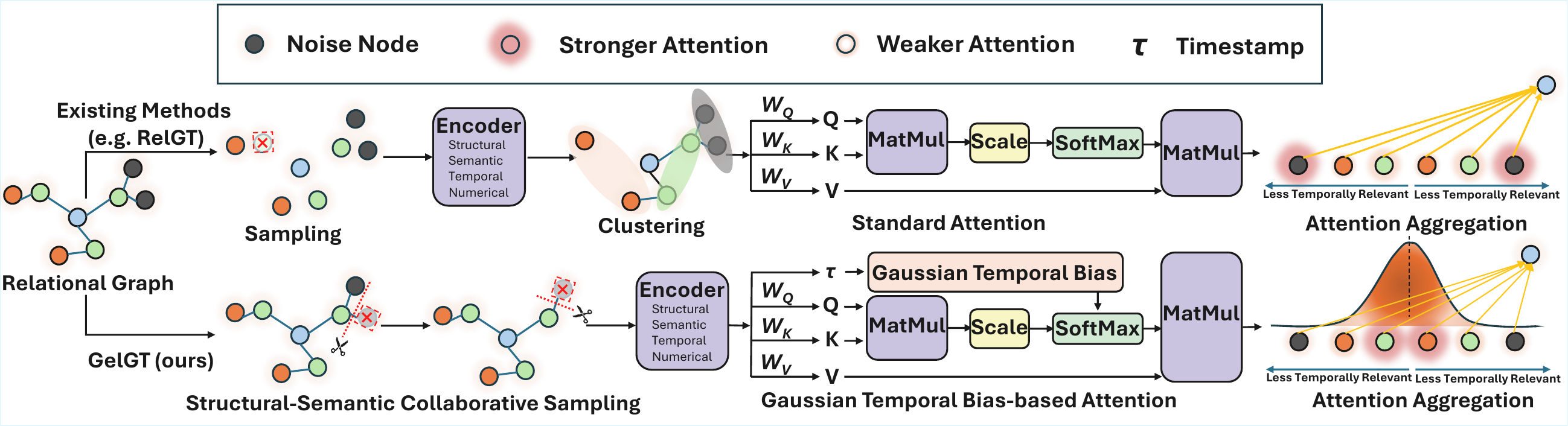}
    \caption{\small \textbf{Comparison of Existing Methods and Our Proposed GelGT.} {\small \textbf{Structurally}, GelGT alleviates structural fragmentation by preserving node connectivity during sampling. \textbf{Semantically}, it mitigates the interference of semantically noisy nodes during sampling. \textbf{Temporally}, it employs \textit{Gaussian Temporal Bias} to explicitly distinguish between temporally relevant and noisy nodes.}}
    \label{fig:intro}
\end{figure*}

Effective sampling must preserve relational structure while avoiding irrelevant information. Existing relational graph learning methods commonly rely on random sampling \citep{ranjan2026relational, 0001RHHHDFLYZHL24, ChenKL25}, which can break connections among sampled entities and lead to structural fragmentation \citep{meher2025insidecorekgevaluatingstructured}. Increasing the sampled neighborhood can alleviate this problem, but may also introduce nodes that contribute little useful information to the prediction \citep{MenczerPS04}. Structural preservation and semantic relevance therefore need to be considered together: structure defines the candidate set, while semantics identifies which distant nodes are worth retaining. 

Beyond structural and semantic relevance, 
the model must determine how the retained information should contribute according to temporal relevance.
Useful evidence may concentrate around a specific temporal region and extend over different temporal ranges. An inaccurate temporal focus may direct attention toward irrelevant information, while an overly narrow temporal range may miss useful dependencies and an overly broad one may introduce temporal noise \citep{GuoZL19}. Effective relational modeling therefore requires both preserving structurally and semantically relevant information during sampling and explicitly distinguishing its temporal relevance during attention.

To address the aforementioned challenges, we propose the {\it \underline{G}aussian R\underline{el}ational \underline{G}raph \underline{T}ransformer} ({\bf GelGT} in short), as shown in Figure~\ref{fig:intro}. 
GelGT first introduces a \textit{structural-semantic collaborative sampling strategy} that preserves relational connectivity while refining distant nodes according to semantic relevance. 
For the retained information, GelGT further introduces a \textit{Gaussian temporal bias} into graph attention, where the Gaussian center and range explicitly characterize \emph{where} and \emph{how broadly} the model attends over time. 
We provide theoretical analyses characterizing the structural preservation, semantic refinement, and temporal discrimination of these mechanisms.

Our main contributions are summarized as follows: (1) We propose \textbf{Gaussian Relational Graph Transformer (GelGT)}, a relational graph transformer that jointly addresses structural, semantic, and temporal challenges in relational graph learning through collaborative subgraph sampling and Gaussian graph attention.
(2) We theoretically show that structure-semantic collaborative sampling preserves the vast majority ($\geq99\%$) of effective structural information and can effectively mitigate semantic noise in seed node embeddings.
We also prove that Gaussian graph attention assigns higher attention scores to temporally relevant nodes and lower attention scores to temporal noise nodes, compared to conventional attention without Gaussian bias (e.g., RelGT).
(3) Extensive experiments on \textbf{7} datasets and \textbf{21} tasks demonstrate that GelGT achieves the state-of-the-art performance on both classification and value regression tasks, with up to a \textbf{13.8\%} improvement in predictive performance.

\section{Preliminaries}
\paragraph{Relational Database (RDB)~\citep{Codd70,Codd79}.} A {\it relational database} is formally defined as a collection of tables, denoted by $\mathcal{R} = \{\bm{T}_1, ..., \bm{T}_n \}$, where $n$ represents the total number of tables. 
Each row in $\bm{T}_i$ corresponds to an entity, which is uniquely identified by a primary key and may contain references to entities in other tables through foreign keys, together with entity-specific attributes and associated timestamp information.

\paragraph{Relational Graph~\citep{FeyHHLR0YYL24}.} The structure of a relational database inherently forms a graph representation, referred to as a {\it relational graph}. A relational graph is formally defined as a heterogeneous temporal graph $\mathcal{G} = \{\mathcal{V}, \mathcal{E}, \bm{\phi}, \bm{\psi}, \bm{\tau}\}$, where nodes $\mathcal{V}$ represent entities from database tables, edges $\mathcal{E}$ represent primary-foreign key relationships, $\bm{\phi}$ maps nodes to their respective types based on source tables, $\bm{\psi}$ assigns relation types to edges, and $\bm{\tau}$ captures the temporal dimension through timestamps. 

\paragraph{RDB Task Reformulation.} Based on the relational graph construction, predictive tasks on RDBs can be naturally cast as learning problems on $\mathcal{G}$. Let $\mathcal{Y}$ denote a target attribute column within a table $\bm{T}_i$. Since each row $j$ in $\bm{T}_i$ corresponds to a unique node $v_j \in \mathcal{V}$, predicting a value in the $j$-th row of $\bm{T}_i$ can be formulated as learning a mapping function $f$: $\mathcal{V} \to \mathcal{Y}$ that leverages both node attributes and graph topology. Following previous work \citep{dwivedi2026relational}, we focus on two paradigms: (i) {\it node classification}, where $\mathcal{Y}$ is a finite set of discrete categorical labels, and (ii) {\it node regression}, where $\mathcal{Y} \subseteq \mathbb{R}$ represents a continuous range of numerical values.

\section{Gaussian Relational Graph Transformer (GelGT)}

This section introduces the GelGT, as illustrated in Figure \ref{fig:model}.
GelGT consists of two components: \textit{structure-semantic collaborative sampling} (Steps \ding{172}-\ding{174}) and \textit{Gaussian graph attention model} (Steps \ding{175}-\ding{177}). 
The former constructs a query-specific computation subgraph, which enables the latter to learn temporally aware representations. Sections \ref{section:3.1} and \ref{section:3.2} describe the two components, respectively.

\begin{figure*}[t]
    \centering
    \includegraphics[width=0.9\textwidth]{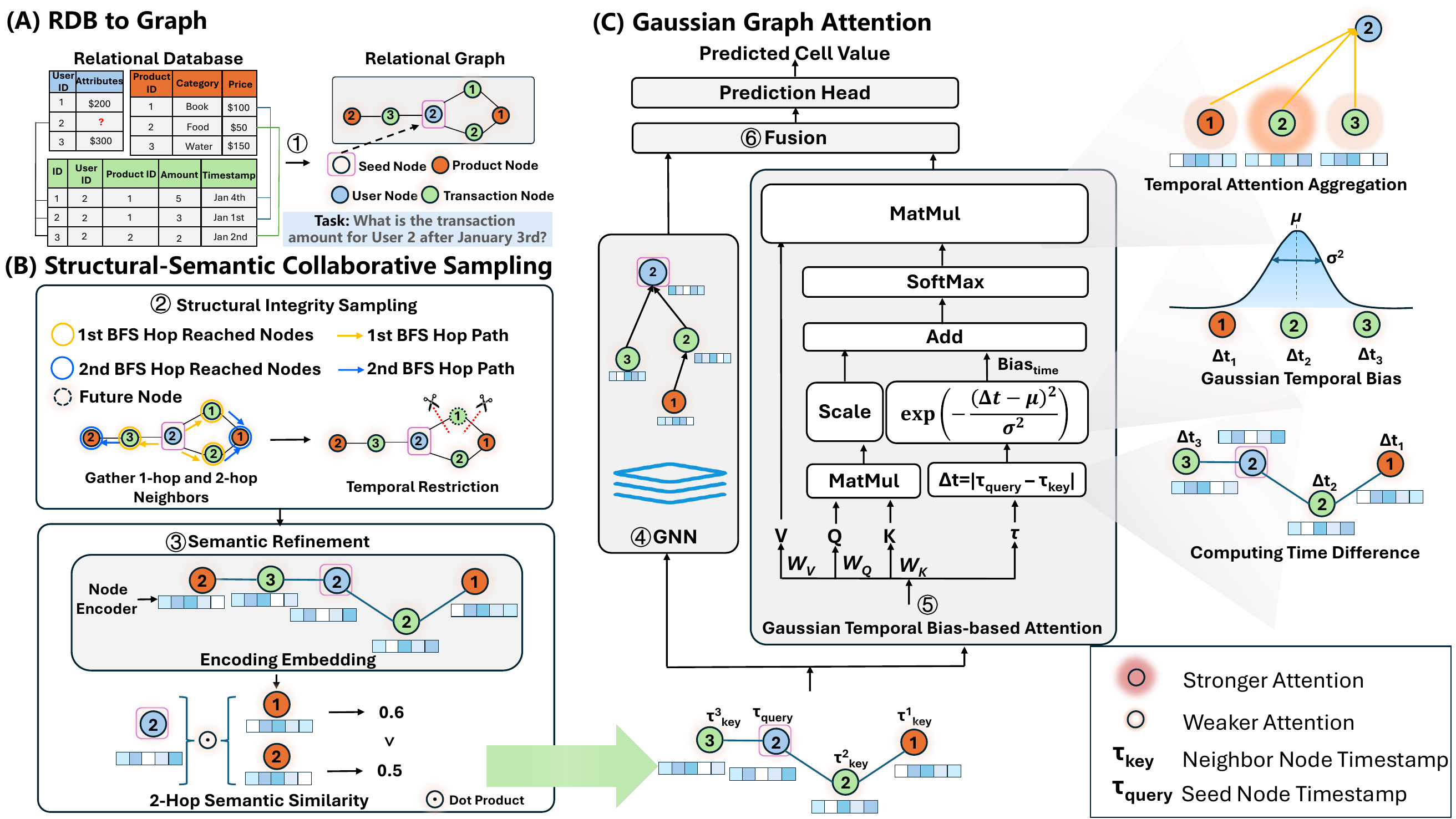}
    \caption{\small \textbf{Overview of the GelGT Framework}. {The process consists of six main steps: \ding{172} \textbf{Graph Construction} converts the relational database into a relational graph. \ding{173} \textbf{Structural Integrity Sampling.} This stage samples subgraphs via BFS while enforcing strict timestamp constraints to mask future nodes for temporal validity. 
    \ding{174} {\bf Semantic Refinement.} It filters noise by retaining only neighbors with high semantic similarity, calculated through dot products of encoded features. \ding{175} A \textbf{GNN} module aggregates topological structural features. \ding{176} \textbf{Gaussian Temporal Bias-based Attention.} It leverages a Gaussian temporal bias to assign high attention scores to temporally relevant nodes while allocating low scores to noisy nodes. 
    \ding{177} Finally, the \textbf{Fusion} module adaptively integrates the outputs of the GNN and the attention via a learnable weight \citep{liduetgraph}.}}
    \label{fig:model}
\end{figure*}

\subsection{Structure-Semantic Collaborative Sampling}
\label{section:3.1}
The goal of our sampling strategy is to construct a subgraph that preserves structural connectivity while filtering out semantically irrelevant nodes. This yields the final subgraph. 
The resulting subgraph explicitly defines the valid interaction scope 
for the subsequent attention model, and is constructed via two stages: (1) \textit{structural integrity sampling} and (2) \textit{semantic refinement}.

\paragraph{Stage 1: Structural Integrity Sampling.} The goal of this stage is to preserve the local relational structure
around a seed node $v$ under a fixed sampling budget $B_s$. 
To achieve this, we employ breadth-first search (BFS) to expand from $v$. Under the fixed sampling budget $B_s$, BFS expands the neighborhood level by level, preferentially retaining information from nodes that are more closely connected to the seed node before reaching farther regions, thereby preserving the local context around $v$. During the expansion, we strictly 
traverse neighbors that satisfy temporal causality\footnote{\label{ft:tc}Temporal causality \citep{PrabhakarOWAR10} requires that the timestamp of any sampled neighbor precedes that of the query node (i.e., $\tau_{\text{neighbor}} < \tau_{\text{seed}}$), to avoid temporal information leakage.}, and stop once the sampling budget $B_s$ is reached. 
This process yields a sampled node set. 
Finally, we construct the subgraph induced by the sampled node set.
Formally, we define the sampled node set $\mathcal{N}_{\textrm{sampled}}(v)$ and the resulting sampled subgraph $\mathcal{G}_{\mathrm{sub}}(v)$ as:
\begin{align}
    &\mathcal{N}_{\mathrm{sampled}}(v) = \{ u_1, u_2, \dots, u_B \}, \quad u_1 = v \\
    &\mathcal{G}_{\mathrm{sub}}(v) = (\mathcal{N}_{\mathrm{sampled}}(v), \mathcal{E}_{\mathrm{sub}}),
\end{align}
where $\mathcal{E}_{\mathrm{sub}} = \{(i, j) \in \mathcal{E} \mid i, j \in \mathcal{N}_{\mathrm{sampled}}(v)\}$ and $\mathcal{E}$ denotes the edge set of the relational graph, and $(u_1, u_2, \dots)$ is the BFS visitation order over {nodes whose timestamps do not exceed the timestamp of node $v$}.

\paragraph{Theoretical Analysis.} We theoretically show that the structural integrity sampling effectively preserves the structural information of the relational graph. To quantify this, we use Katz centrality \citep{katz1953new} as the topological metric and denote the Katz centrality computed on the full graph and the sampled subgraph as $C^{\mathcal{G}}(v)$ 
and $C^{\mathcal{G}_{\mathrm{sub}}}(v)$, respectively.
We establish the quantitative relationship between the structural loss $\Delta C(v) = | C^{\mathcal{G}}(v) - C^{\mathcal{G}_{\mathrm{sub}}}(v) |$ arising from structural integrity sampling and the full-graph structural information in Theorem \ref{lab:them1}.

\begin{theorem}[\textbf{Upper Bound of Relative Structural Loss}]
\label{lab:them1}
The ratio of the structural loss $\Delta C(v) = | C^{\mathcal{G}}(v) - C^{\mathcal{G}_{\mathrm{sub}}}(v) |$ arising from structural integrity sampling to the full-graph structural information $C^{\mathcal{G}}(v)$ satisfies:
\begin{equation}
    \frac{|\Delta C(v)|}{C^{\mathcal{G}}(v)} \leq 0.01
\end{equation}
\end{theorem}

We provide a detailed proof of Theorem \ref{lab:them1} in Appendix \ref{lab:a.1}. Theorem \ref{lab:them1} demonstrates that the structural integrity sampling preserves at least 99\% of the effective structural information, ensuring that the sampled subgraph provides a reliable structural basis for subsequent modeling.

\paragraph{Stage 2: Semantic Refinement.} Then, we refine the sampled subgraph $\mathcal{G}_\mathrm{sub}(v)$ constructed in Stage 1 to mitigate semantic noise. 
Specifically, we first encode the raw attributes of all nodes in the subgraph $\mathcal{G}_\mathrm{sub}(v)$ to low-dimensional embeddings $\mathbf{h}$. 
{We compute the semantic similarity between the seed node $v$ and each candidate node $u$ beyond its one-hop neighborhood via dot product:
\begin{equation}
    s(v, u) = \mathbf{h}_v^\top \mathbf{h}_u, \quad \forall u \in \mathcal{N}_{r} = \mathcal{N}_{\mathrm{sampled}}(v) \setminus (\{v\} \cup \mathcal{N}_1(v)),
\end{equation}

where $\mathcal{N}_1(v)$ denotes the set of 1-hop
neighbors of $v$ satisfying the temporal causality constraint\footref{ft:tc}. 
Among the remaining nodes ($\mathcal{N}_{r}$) in the sampled subgraph, excluding $v$ and the nodes in $\mathcal{N}_1(v)$, we compute semantic similarity scores and use them to guide node selection.
Given the final training budget $B_t < B_s$, we first include the  nodes in $\mathcal{N}_1(v)$, and then select additional nodes from $\mathcal{N}_{r}$ in descending order of semantic similarity until the budget $B_t$ is reached.
We prioritize one-hop neighbors because the structural contribution of $k$-hop neighbors decays with hop distance:
\begin{equation}
\label{eq:ckobeta}
    \Delta \mathcal{C}_{k\text{-hop}}(v)
    = \mathcal{O}(\beta^k)\footnote{A detailed derivation of Equation~\ref{eq:ckobeta} is provided in Appendix~\ref{lab:lemma3}},
\end{equation}
where $0<\beta<1$ is the attenuation factor. Hence, one-hop neighbors provide the most important local structural information, motivating their preservation during semantic refinement. Finally, we construct the training subgraph $\mathcal{G}_{\mathrm{final}}$ from these $B_t$ selected training nodes and their edges to $v$ that satisfy the temporal causality constraint\footref{ft:tc}.}

\paragraph{Theoretical Analysis.} Since noise nodes tend to exhibit lower semantic similarity to the seed node, semantic refinement excludes them from the final sampled subgraph $\mathcal{G}_{\mathrm{final}}$. 
Let $\mathcal{I}$ denote the aggregated information from neighbors of the seed node, decomposed into semantically relevant information $\mathcal{I}_\text{relevance}$ and noise $\mathcal{I}_\text{noise}$. 
Semantic refinement therefore increases the relevance-to-noise ratio (by amplifying $\mathcal{I}_\text{relevance}$ and suppressing $\mathcal{I}_\text{noise}$), as formalized in Theorem \ref{lab:theorem2}.
\begin{theorem}[\textbf{Enhancement of Semantic Relevance-to-Noise Ratio}]
    \label{lab:theorem2}
    The semantic relevance-to-noise ratio after refinement strictly exceeds that of the pre-refinement stage. Formally,
    \begin{equation}
    \frac{\mathcal{I}^\mathrm{before}_{\mathrm{relevance}}}{\mathcal{I}^\mathrm{before}_{\mathrm{noise}}} < \frac{\mathcal{I}^\mathrm{after}_{\mathrm{relevance}}}{\mathcal{I}^\mathrm{after}_{\mathrm{noise}}}
    \end{equation}
    where $\{\mathcal{I}^\mathrm{before}_{\mathrm{relevance}}, \mathcal{I}^\mathrm{before}_{\mathrm{noise}}\}$ and $\{\mathcal{I}^\mathrm{after}_{\mathrm{relevance}}, \mathcal{I}^\mathrm{after}_{\mathrm{noise}}\}$ denote the semantic relevance and semantic noise before and after the refinement process, respectively.
\end{theorem}

Theorem~\ref{lab:theorem2} demonstrates that our proposed semantic refinement strategy can effectively mitigate semantic noise in {seed node embeddings}.
A detailed proof is in Appendix \ref{a.3}.

Through structural integrity sampling and semantic refinement, we obtain the final sampled subgraph $\mathcal{G}_\mathrm{final}$. 
The resulting subgraph defines the set of structurally connected and semantically refined nodes used by the subsequent GNN and attention modules.

\subsection{Gaussian Temporal Bias-based Attention}
\label{section:3.2}

The sampling stage determines which nodes should be retained based on structural connectivity and semantic relevance. We next consider how the retained information should contribute to prediction according to its temporal relevance. As shown in Figure~\ref{fig:model}, GelGT employs a composite architecture that combines a GNN with graph attention. The GNN aggregates local topological information, while the attention module captures dependencies among the retained nodes. To explicitly model their temporal relevance, we introduce an adaptive Gaussian temporal bias into the attention mechanism. 

\paragraph{Adaptive Gaussian Temporal Bias.} 
A straightforward strategy for capturing temporal information, as in RelGT \citep{dwivedi2026relational}, is to encode timestamps jointly with other node attributes. 
However, such integrated representations entangle temporal information with structural and semantic information, making it difficult for the attention mechanism to explicitly distinguish temporally relevant nodes from noise. 
To address this, we introduce a learnable attention bias that parameterizes a temporal preference distribution via a continuous Gaussian function, allowing the attention mechanism to 
selectively focus on relevant temporal information.

Formally, given the input feature matrix $\mathbf{H} \in \mathbb{R}^{N \times d}$ (derived from the encoders as shown in Appendix \ref{lab:encoder}), we define the standard linear projections for the Query ($\mathbf{Q}$), Key ($\mathbf{K}$), and Value ($\mathbf{V}$) matrices as:
$\mathbf{Q} = \mathbf{W}_Q \mathbf{H}, \quad \mathbf{K} = \mathbf{W}_K \mathbf{H}, \quad \mathbf{V} = \mathbf{W}_V \mathbf{H}$,
where $\mathbf{W}_Q, \mathbf{W}_K, \mathbf{W}_V$ are learnable parameter matrices. 
Different from vanilla self-attention \citep{VaswaniSPUJGKP17}, we explicitly inject a temporal bias $\text{Bias}_{\text{time}}$ into the {scaled attention score ($\frac{\mathbf{Q}\mathbf{K}^T}{\sqrt{d}}$)} before normalization, as shown in Equation~\ref{eq:attention_with_bias}. 
\begin{equation}
    \label{eq:attention_with_bias}
    \text{Attention}(Q, K, V) = \text{Softmax}\left( \frac{\mathbf{Q}\mathbf{K}^T}{\sqrt{d}} + \text{Bias}_{\text{time}} \right) \mathbf{V}.
\end{equation} 
The adaptive temporal bias ($\text{Bias}_{\text{time}}$) is formally defined as:
\begin{equation}
    \label{eq:gaussian_bias}
    \text{Bias}_{\text{time}} = \text{Linear}\left( \exp\left( -\frac{(\Delta t - \mu)^2}{\sigma^2} \right) \right).
\end{equation}
$\Delta t$ denotes the time interval between query and key nodes. Mean $\mu$ and standard deviation $\sigma$ are learnable parameters.
{Temporal long-range perception is critical for relational graph learning \citep{lachi2025integrating}.
The learned $\mu$ serves to identify the most temporally relevant information (i.e., temporal center) within the long-range temporal receptive field of a training node.
Besides, the temporal relevance information decays with increasing distance from the temporal center \citep{GuoZL19}.
Thus, the learned $\sigma$ models the information decay effect, making the temporal bias align with the underlying temporal relevance distribution.
}

\paragraph{Theoretical Analysis.} In GelGT, the distribution of attention weights is jointly defined by {the temporal center} $\mu$ and scope $\sigma$. 
This parameterized structure achieves a systematic separation between temporally relevant nodes and temporal noise. 
{First, Lemma \ref{lem:conv_center} proves that the learnable parameter $\mu$ converges to the ground-truth relevant timestamp. Based on that, Theorem \ref{theo:strict_attention} further proves that our method assigns a higher attention score to temporally relevant nodes than to noise nodes.}

\begin{lemma}[\textbf{Convergence of Temporal Center}]
\label{lem:conv_center}
Let $t^*$ be the ground-truth relevant timestamp. The gradient of the Gaussian temporal bias $\mathcal{B}(t^*)$ with respect to the learnable temporal center $\mu$ satisfies:
$\frac{\partial \mathcal{B}}{\partial \mu} = \mathcal{B}(t^*) \cdot \frac{t^* - \mu}{\sigma^2}$. 
The term $(t^* - \mu)$ drives $\mu$ toward $t^*$ during optimization.
\end{lemma}
Detailed proof of Lemma \ref{lem:conv_center} is in Appendix \ref{lab:a.4}. Based on the convergence property established in Lemma \ref{lem:conv_center}, {we further quantify the disparities in attention score assignment between temporally relevant nodes and noise nodes, specifically comparing the mechanisms w/ and w/o Gaussian bias.}

{\begin{theorem}[\textbf{Attention Score Ratio Enhancement}]
\label{theo:strict_attention}
Let $\{\alpha^A_{relevant}, \alpha^A_{noise}\}$ denote the attention scores that the attention mechanism $A$ assigned to temporally related nodes and noise nodes. Then,
\begin{equation}
    \frac{\alpha_{relevant}^\mathrm{w/ \ G}}{\alpha_{noise}^\mathrm{w/ \ G}} = \mathrm{\mathbf{e}} \cdot \frac{\alpha_{relevant}^\text{w/o \ G}}{\alpha_{noise}^\text{w/o \ G}},
\end{equation}
where $\mathrm{\mathbf{e}}$ is Euler's number, ``w/ G'' and ``w/o G'' indicate the attention mechanism with and without the Gaussian bias, respectively.
\end{theorem}}

We provide a detailed proof of Theorem \ref{theo:strict_attention} in Appendix \ref{lab:a.5}. Theorem \ref{theo:strict_attention} demonstrates that, compared to vanilla self-attention \citep{VaswaniSPUJGKP17}, our mechanism increases the ratio of attention weights assigned to temporally relevant versus noisy nodes by a factor of $\mathrm{\mathbf{e}}$.
Consequently, our transformer more effectively discriminates between temporally relevant information and temporal noise. 

\section{Experiments}
\label{experiment}
In this section, we conduct extensive experiments to answer the following research questions: \\
\textbf{RQ1}: Does GelGT outperform existing state-of-the-art methods in terms of predictive performance? 
\textbf{RQ2}: How much does each component of GelGT contribute to its performance, and how well can the proposed components be integrated into other baseline frameworks?\\
\textbf{RQ3}: How do different configurations of GelGT affect its performance?\\
\textbf{RQ4}: How efficient is GelGT in terms of training and inference time?

\subsection{Experimental Setup}
\textbf{Datasets} GelGT is evaluated on the recently introduced RDL Benchmark (RelBench) \citep{0001RHHHDFLYZHL24}. In total, RelBench contains \textbf{30} tasks across \textbf{7} datasets, covering node classification, node regression, and recommendation. For our evaluation, we focus on \textbf{21} tasks on node classification and regression, following the task settings used in RelGT~\citep{dwivedi2026relational}.\\
\textbf{Baselines} We compare GelGT with six state-of-the-art baselines for node classification and node regression tasks (Tables \ref{tab:tab.1} and \ref{tab:tab.2}). 
Our baselines include relational graph transformer (RelGT) \citep{dwivedi2026relational}, as well as the {graph-centric relational database foundation model} Griffin \citep{WangWGWYWZ25}. 
Consistent with the experimental settings in prior work \citep{dwivedi2026relational}, we also evaluate two variants of the Heterogeneous Graph Transformer (HGT) \citep{HuDWS20} to further examine the effectiveness of GelGT relative to existing graph transformers. 
In addition, we compare against LightGBM~\citep{KeMFWCMYL17}, a strong baseline from the tabular learning literature.\\
\textbf{Evaluation Metrics} Node-level classification aims to predict binary labels for a given node at
a specific seed time. 
We use the Area Under the ROC Curve (AUC) \citep{1983A} as the evaluation metric. 
Node-level regression involves predicting numerical labels for a node at a given seed time. We use Mean Absolute Error (MAE) as the evaluation metric.

\begin{table*}[t!]
    \centering
     \caption{\small \textbf{Test Set Results on the Node Classification Tasks in RelBench \citep{0001RHHHDFLYZHL24}.} {Best values (based on mean performance) are in \textbf{bold}, and the second-best values are \underline{underlined}. Results are reproduced based on available code. $\uparrow$: higher is better.}}
     \vspace{-10pt}
    \begin{adjustbox}{width=\textwidth}
    
    \begin{tabular}{lcccccccc}
  
	\toprule
	\bf Dataset & \bf Task & \bf RDL & \bf HGT & \bf HGT+PE & \bf Griffin & \bf RelGT & \bf LightGBM & \bf Ours\\
	\midrule

    rel-f1 & driver-dnf ($\uparrow$) & 0.7262$\pm$0.0027 & 0.7077$\pm$0.0153 & 0.7117$\pm$0.0084 & 0.7091$\pm$0.0045  & \underline{0.7587$\pm$0.0413} & 0.6886$\pm$0.0034 & \textbf{0.7608$\pm$0.0175} \\
	rel-f1 & driver-top3 ($\uparrow$) & 0.7554$\pm$0.0154 & 0.7765$\pm$0.0066 & 0.7340$\pm$0.0018 & 0.7795$\pm$0.0139 & 0.8352$\pm$0.0342 & 0.7393$\pm$0.0121  & \textbf{0.8408$\pm$0.0063 }\\
	rel-avito & user-clicks ($\uparrow$) & 0.6590$\pm$0.0195 & 0.6376$\pm$0.0298 & 0.6457$\pm$0.0099 & 0.6330$\pm$0.0025 & \underline{0.6830$\pm$0.0602} & 0.5360$\pm$0.0032 & \textbf{0.6844$\pm$0.0600} \\
	rel-avito & user-visits ($\uparrow$) & 0.6620$\pm$0.0010 & 0.6432$\pm$0.0002 & 0.6495$\pm$0.0022 & 0.6468$\pm$0.0045 & \underline{0.6678$\pm$0.0015} & 0.5305$\pm$0.0011 & \textbf{0.6695$\pm$0.0060}\\
	
	rel-event & user-repeat ($\uparrow$) & 0.7689$\pm$0.0159 & 0.6496$\pm$0.0220 & 0.6536$\pm$0.0137 & \underline{0.7709$\pm$0.0062} & 0.7609$\pm$0.0219 & 0.5305$\pm$0.0211  & \textbf{0.8357$\pm$0.0203} \\
	rel-event & user-ignore ($\uparrow$) & 0.8162$\pm$0.0111 & \underline{0.8247$\pm$0.0096} & 0.8161$\pm$0.0007 & 0.8153$\pm$0.0038 & 0.8157$\pm$0.0040 & 0.7993$\pm$0.0154  & \textbf{0.8779$\pm$0.0122} \\
	rel-trial & study-outcome ($\uparrow$) & 0.6860$\pm$0.0101 & 0.5837$\pm$0.0141 & 0.5921$\pm$0.0303 & 0.6908$\pm$0.0071 & 0.6861$\pm$0.0040 & \underline{0.7009$\pm$0.0063}  & \textbf{0.7254$\pm$0.0059} \\
	rel-amazon & user-churn ($\uparrow$) & \underline{0.7042$\pm$0.0005} & 0.6643$\pm$0.0041 & 0.6619$\pm$0.0042 & 0.7000$\pm$0.0036 & 0.7039$\pm$0.0040 & 0.5222$\pm$0.0051 & \textbf{0.7050$\pm$0.0065} \\
	rel-amazon & item-churn ($\uparrow$) & \underline{0.8281$\pm$0.0003} & 0.7797$\pm$0.0039 & 0.7803$\pm$0.0053 & 0.8110$\pm$0.0051 & 0.8255$\pm$0.0006 & 0.6254$\pm$0.0044  & \textbf{0.8297$\pm$0.0096} \\
	rel-stack & user-engagement ($\uparrow$) & 0.9021$\pm$0.0007 & 0.8847$\pm$0.0044 & 0.8817$\pm$0.0046 & 0.8980$\pm$0.0023 & \underline{0.9053$\pm$0.0005} & 0.6339$\pm$0.0350 & \textbf{0.9086$\pm$0.0186} \\
	rel-stack & user-badge ($\uparrow$) & \underline{0.8986$\pm$0.0008} & 0.8608$\pm$0.0044 & 0.8566$\pm$0.0068 & 0.8700$\pm$0.0031 & 0.8632$\pm$0.0018 & 0.6343$\pm$0.0079  & \textbf{0.9044$\pm$0.0086} \\
	rel-hm & user-churn ($\uparrow$) & \underline{0.6988$\pm$0.0021} & 0.6695$\pm$0.0067 & 0.6569$\pm$0.0109 & 0.6804$\pm$0.0014 & 0.6927$\pm$0.0019 & 0.5521$\pm$0.0025  & \textbf{0.6997$\pm$0.0017} \\
    
    \bottomrule
    \end{tabular}
    \end{adjustbox}
  \label{tab:tab.1}
  \vspace{-10pt}
\end{table*}
\begin{table*}[t!]
    \centering
     \caption{\small \textbf{Test Set Results on the Node Regression Tasks in RelBench \citep{0001RHHHDFLYZHL24}.} Best values (based on mean performance) are in \textbf{bold}, and the second-best values are \underline{underlined}. {Results are reproduced based on available code. Results marked with $^\dagger$ are reproduced using the official code. $\downarrow$: lower is better.}}
     \vspace{-10pt}
    \begin{adjustbox}{width=\textwidth}
    
    \begin{tabular}{lcccccccc}
  
	\toprule
	\bf Dataset & \bf Task & \bf RDL  & \bf HGT & \bf HGT+PE & \bf Griffin  & \bf RelGT  & \bf LightGBM  & \bf Ours\\
	\midrule

    rel-f1 & driver-position ($\downarrow$) & 4.022$\pm$0.119 & 4.2263$\pm$0.0580 & 4.3921$\pm$0.1382 & \underline{3.9315$\pm$0.2903} & 3.9170$\pm$0.3448 $^\dagger$ & 4.1700$\pm$0.2200& \textbf{3.7345$\pm$0.1200} \\
	rel-avito & ad-ctr ($\downarrow$) & \underline{0.0410$\pm$0.0010} & 0.0462$\pm$0.0021 & 0.0483$\pm$0.0027 & 0.0433$\pm$0.0004 & 0.0420$\pm$0.0009 $^\dagger$ & 0.0410 $\pm$0.0041 & \textbf{0.0362$\pm$0.0009} \\
	
	rel-event & user-attendance ($\downarrow$) & 0.258$\pm$0.006 & 0.2635$\pm$0.0000 & 0.2611$\pm$0.0043 & 0.3708$\pm$0.0094 & \underline{0.2502$\pm$0.0033} & 0.2640$\pm$0.0011 & \textbf{0.2423$\pm$0.0031} \\
	rel-trial & study-adverse ($\downarrow$) & 44.473$\pm$0.209 & 45.1692$\pm$2.6927 & \underline{42.6484$\pm$0.2785} & 59.0902$\pm$0.0022 & 43.9923$\pm$0.5928 & 44.0110$\pm$0.2020 & \textbf{42.5699$\pm$0.9700} \\
    rel-trial & site-success ($\downarrow$) & 0.4000$\pm$0.0200 & 0.4428$\pm$0.0047 & 0.4396$\pm$0.0083 & 0.3801$\pm$0.0041 & \underline{0.3761$\pm$0.0178} $^\dagger$ & 0.4250$\pm$0.0008 & \textbf{0.3485$\pm$0.0200} \\
	rel-amazon & user-ltv ($\downarrow$) & 14.3130$\pm$0.013 & 15.4120$\pm$0.0447 & 15.8643$\pm$0.0924 & 19.6055$\pm$0.0239 & \underline{14.2665$\pm$0.0154} & 16.7830$\pm$0.0188 & \textbf{14.2538$\pm$0.0189} \\
	rel-amazon & item-ltv ($\downarrow$) & \underline{50.0530$\pm$0.1630} & 55.8683$\pm$0.6003 & 55.8493$\pm$0.3226 & 66.7938$\pm$0.0193 & 54.6626$\pm$0.7006 $^\dagger$ & 60.5690$\pm$0.0033 & \textbf{49.2109$\pm$0.8092} \\
	rel-stack & post-votes ($\downarrow$) & \underline{0.0650$\pm$0.0000} & 0.0679$\pm$0.0000 & 0.0680$\pm$0.0000 & 0.1395$\pm$0.0076 & 0.0654$\pm$0.0002 & 0.0680$\pm$0.0065 & \textbf{0.0648$\pm$0.0025} \\
	rel-hm & item-sales ($\downarrow$) & \textbf{0.0560$\pm$0.0000} & 0.0641$\pm$0.0012 & 0.0639$\pm$0.0003 & 0.0597$\pm$0.0008 & \underline{0.0563$\pm$0.0006} $^\dagger$ & 0.0760$\pm$0.0002 & \textbf{0.0560$\pm$0.0009} \\
    
    \bottomrule
    \end{tabular}
    \end{adjustbox}

  \label{tab:tab.2}
  \vspace{-10pt}
\end{table*}

\begin{table*}[t]
    \centering
    
    \begin{minipage}{0.48\linewidth}
        \centering
        \caption{\small \textbf{Performance on Node Classification Tasks on the SALT Dataset.} \small$\uparrow$: higher is better.}
        \vspace{-10pt}
        \label{tab:salt_dataset}
        \begin{adjustbox}{width=\linewidth}
        \begin{tabular}{lccccc}
        \toprule
        \bf Task $\uparrow$ & \bf RDL & \bf HGT & \bf RelGT & \bf RGP \citep{lachi2025integrating} & \bf Ours \\
        \midrule
        rel-salt-sales-group & 0.20 & 0.20 & 0.35 & 0.34 & \textbf{0.38} \\
        rel-salt-sales-incoterms & 0.70 & 0.75 & 0.73 & - & \textbf{0.77} \\
        rel-salt-sales-payterms & 0.39 & 0.60 & 0.61 & 0.58 & \textbf{0.64} \\
        rel-salt-sales-office & 0.93 & 0.96 & 0.98 & - & \textbf{0.99} \\
        rel-salt-sales-shipcond & 0.59 & 0.76 & 0.75 & 0.81 & \textbf{0.83} \\
        rel-salt-item-incoterms & 0.64 & 0.75 & 0.84 & 0.81 & \textbf{0.87} \\
        rel-salt-item-plant & 0.90 & 0.91 & 0.97 & - & \textbf{0.99} \\
        rel-salt-item-shippoint & 0.94 & 0.93 & 0.98 & - & \textbf{0.99} \\
        \bottomrule
        \end{tabular}
        \end{adjustbox}
    \end{minipage}\hfill
    \begin{minipage}{0.48\linewidth}
        \centering
        \caption{\small \textbf{Performance Comparison with More Representative Baselines.} \small$\uparrow$: higher is better.}
        \vspace{-10pt}
        \label{tab:emerging_baselines}
        \begin{adjustbox}{width=\linewidth}
        \begin{tabular}{lccccc}
        \toprule
        \bf Task $\uparrow$ & \bf Rel-LLM & \bf Nodeformer & \bf DUALFormer & \bf Ours \\
        \midrule
        rel-f1-driver-top3 & 0.82 & 0.82 & 0.79 & \textbf{0.84} \\
        rel-avito-user-clicks & 0.66 & 0.66 & 0.65 &\textbf{0.68} \\
        rel-event-user-repeat & 0.79 & 0.79 & 0.74 &\textbf{0.83} \\
        rel-event-user-ignore & 0.83 & 0.83 & 0.82 &\textbf{0.87} \\
        rel-trial-study-outcome & 0.71 & 0.71 & 0.69 &\textbf{0.72} \\
        rel-stack-user-badge & 0.89 & 0.89 & 0.87 &\textbf{0.90} \\
        rel-hm-user-churn & 0.55 & 0.61 & 0.59 &\textbf{0.69} \\
        \bottomrule
        \end{tabular}
        \end{adjustbox}
    \end{minipage}
    \vspace{-12pt}
\end{table*}
\begin{table*}[t]
    \centering
    \caption{\small \textbf{Ablation Study of GelGT Components on Node Classification Tasks.} {\small$\uparrow$: higher is better.}}
    \label{tab:tab.3}
    \vspace{-10pt}
    \renewcommand{\arraystretch}{1.0} 
    \setlength{\tabcolsep}{2pt} 
    
    \begin{adjustbox}{width=\textwidth}
    \begin{tabular}{lccccc}
    \toprule
    \bf Dataset & \bf Task & \bf GelGT (Ours) & \bf w/o Structural Integrity Sampling & \bf w/o Semantic Refinement & \bf w/o Adaptive Gaussian Bias \\
    \midrule
    rel-avito & user-clicks ($\uparrow$)& {0.6844$\pm$0.0600} & {0.6688$\pm$0.0057} & 0.6677$\pm$0.0031 & 0.6669$\pm$0.0043 \\
    rel-avito & user-visits ($\uparrow$)& {0.6695$\pm$0.0060} & 0.6612$\pm$0.0018 & {0.6652$\pm$0.0040} & 0.6630$\pm$0.0028 \\
    rel-event & user-ignore ($\uparrow$)& {0.8779$\pm$0.0122} & 0.8691$\pm$0.0073 & {0.8699$\pm$0.0086} & 0.8638$\pm$0.0091 \\
    rel-trial & study-outcome ($\uparrow$)& {0.7254$\pm$0.0059} & 0.7161$\pm$0.0060 & 0.7155$\pm$0.0040 & {0.7192$\pm$0.0076} \\
    rel-amazon & user-churn ($\uparrow$)& {0.7050$\pm$0.0065} & 0.6812$\pm$0.0080 & {0.6973$\pm$0.0095} & 0.6945$\pm$0.0089 \\
    \bottomrule
    \end{tabular}
    \end{adjustbox}
    \vspace{-12pt}
\end{table*}
\begin{table*}[!t]
    \centering
    \caption{\small \textbf{Ablation Study of GelGT Components on Node Regression Tasks.} {\small$\downarrow$: lower is better.}}
    \vspace{-10pt}
    \label{tab:tab.4}
    
    \renewcommand{\arraystretch}{1.0} 
    \setlength{\tabcolsep}{2pt} 
    
    \begin{adjustbox}{width=\textwidth}
    \begin{tabular}{lccccc}
    \toprule
    \bf Dataset & \bf Task & \bf GelGT (Ours) & \bf w/o Structural Integrity Sampling & \bf w/o Semantic Refinement & \bf w/o Adaptive Gaussian Bias \\
    \midrule
	rel-avito & ad-ctr ($\downarrow$) & 0.0362$\pm$0.0009 & 0.0372$\pm$0.0009 & 0.0368$\pm$0.0007 & 0.0370$\pm$0.0004 \\
	rel-trial & site-success ($\downarrow$) & 0.3485$\pm$0.0200 & 0.3909$\pm$0.0048 & 0.3491$\pm$0.0178 & 0.3583$\pm$0.0034 \\
	rel-hm & item-sales ($\downarrow$) & 0.0560$\pm$0.0009 & 0.0584$\pm$0.0021 & 0.0599$\pm$0.0014 & 0.0585$\pm$0.0008 \\
    \bottomrule
    \end{tabular}
    \end{adjustbox}
    \vspace{-10pt}
\end{table*}
\subsection{Performance (\textbf{RQ1})}
We evaluate GelGT on \textbf{12} node classification tasks across \textbf{7} datasets, as shown in Table~\ref{tab:tab.1}. The results demonstrate the strong performance of GelGT compared to baselines. Specifically, GelGT  improves AUC by up to \textbf{6.2}\% over the baselines. It ranks first in AUC on every evaluated task, indicating its strong ability to effectively distinguish positive nodes from negative ones. Besides, we evaluate GelGT on \textbf{9} node regression tasks across \textbf{7} datasets, as shown in Table \ref{tab:tab.2}. The results demonstrate the best performance of GelGT compared to baselines, achieving up to a \textbf{13.8}\% improvement in MAE. 
Furthermore, on the very large-scale SALT dataset from RelBench~v2 \citep{gu2026relbenchv2largescalebenchmark}, GelGT achieves state-of-the-art results across all \textbf{8} classification tasks, shown in Table~\ref{tab:salt_dataset}. GelGT consistently outperforms a broader set of representative baselines that incorporate recent advances across many evaluated tasks, including the LLM-based approach Rel-LLM \citep{WuDL25}, and transformer-based long-range graph models such as Nodeformer \citep{WuZLWY22} and DUALFormer~\citep{ZhuoLLMF0GWC025}, as shown in Table~\ref{tab:emerging_baselines}.
Additional details on these baselines and the SALT dataset, along with a comparison with RT~\citep{ranjan2026relational}, are provided in Appendix~\ref{sec:apadditionalbaseline}.

\subsection{Component Study (\textbf{RQ2})}
First, we perform an ablation study by removing key components of GelGT:
\textbf{(1) w/o Structural Integrity Sampling}: We use random sampling instead. This does not preserve the topological structure of the relational graph.
\textbf{(2) w/o Semantic Refinement}: We retain all sampled nodes, including the irrelevant ones.
\textbf{(3) w/o Adaptive Gaussian Temporal Bias}: We remove the adaptive Gaussian temporal bias, reducing the attention mechanism to vanilla self-attention \citep{VaswaniSPUJGKP17}.
Tables \ref{tab:tab.3} and \ref{tab:tab.4} present the ablation study results for the node classification and node regression tasks, respectively. In Tables \ref{tab:tab.3} and \ref{tab:tab.4}, removing structural integrity sampling degrades performance, confirming the necessity of the topological structure preserved by our sampling strategy. Eliminating semantic refinement leads to a clear drop, demonstrating that GelGT effectively removes irrelevant information. Finally, excluding the adaptive Gaussian temporal bias results in an observable performance loss, indicating that simply encoding time as features is insufficient.

\begin{wrapfigure}{r}{0.4\textwidth}
\vspace{-10pt}
    \centering
        \includegraphics[width=\linewidth, height=2.5cm]{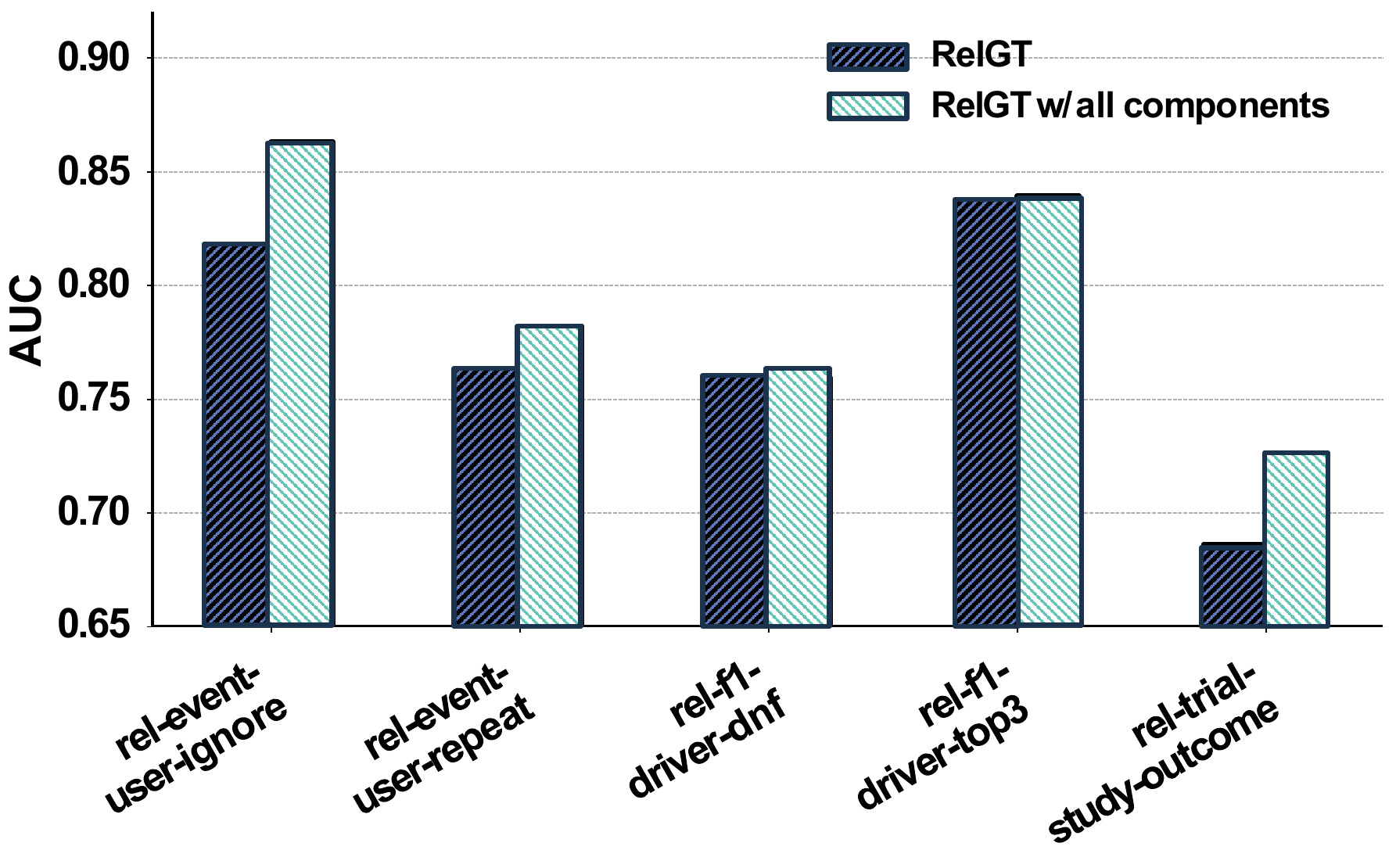}
        \vspace{-22pt}
            \caption{\small \textbf{RelGT w/ all GelGT Components vs. RelGT}}
        \label{fig:relgtcompare-new}
        \vspace{-10pt}
\end{wrapfigure}
Besides, we evaluate the performance of RelGT enhanced with our proposed key components, including structural-semantic collaborative sampling and Gaussian temporal bias, across five different tasks to validate their effectiveness. As shown in Figure \ref{fig:relgtcompare-new}, this uniform performance gain across diverse tasks demonstrates that our approach effectively addresses the limitations of the original framework. In summary, the results demonstrate that our proposed modules effectively generalize to other baseline frameworks, significantly enhancing their predictive performance. In Appendix \ref{sampleing}, we further demonstrate through experiments that the proposed sampling strategy is model-agnostic and can improve the predictive performance of RelGNN~\citep{ChenKL25}.

What is more, we conduct more fine-grained ablation studies to validate the rationale behind several of our design choices, including the GNN branch, selection rule, traversal strategy, and temporal bias as follows.
\textbf{(1) GNN Branch.} As shown in  Table \ref{tab:gnn_ablation}, performance consistently drops across all datasets. This validates the necessity of the GNN branch in GelGT. The attention branch, even with hop-distance embeddings and GNN-based positional encodings, operates on a fully connected sampled subgraph and therefore only captures relative positional information, without explicitly modeling the true PK-FK edges. In contrast, the GNN branch performs message passing directly along PK-FK edges, which preserves the original relational structure and better captures the semantic dependencies encoded in these edges.
\textbf{(2) Selection Rule.} Top-$k$ selection by the dot-product score beats both random-$k$ and bottom-$k$ drawn from the same candidate pool on all five tasks (Table~\ref{tab:design_choices} (a)), confirming that the score ranks nodes by downstream usefulness rather than acting as an arbitrary heuristic. 
\textbf{(3) Traversal Strategy.} BFS beats DFS on all five tasks, with the largest relative gain of 2.1\% on study-outcome (Table~\ref{tab:design_choices} (b)). {This advantage stems from their different traversal patterns under a fixed sampling budget. DFS preferentially explores a single PK--FK branch, concentrating samples on a few relational paths and leaving other direct PK--FK neighbors uncovered. In contrast, BFS expands PK--FK neighbors level by level, providing broader relational coverage within the same budget.}
\textbf{(4) Other Temporal Biases.} Our Gaussian bias beats both a scaled linear bias and a learnable temporal embedding on all seven datasets (Table~\ref{tab:design_choices} (c)). 
\textbf{(5) Temporal Bias under Temporal Shuffling.} 
{Under temporal shuffling (i.e., randomly shuffling the timestamps of the dataset to make the evaluation more challenging), w/o temporal bias (i.e., RelGT's attention) experiences an average decline of 0.0441 (5.64\%), while w/ temporal bias (i.e., GelGT) exhibits a smaller drop of 0.0348 (4.36\%) (Table~\ref{tab:design_choices} (d)). This suggests that incorporating temporal bias into the attention mechanism improves its robustness to temporal perturbations.}
\begin{table*}[t]
    \centering
    \caption{\small \textbf{Effect of the GNN Branch on Node Classification Tasks.} {\small$\uparrow$: higher is better.}}
    \vspace{-10pt}
    \label{tab:gnn_ablation}
    
    \renewcommand{\arraystretch}{1.2} 
    \setlength{\tabcolsep}{6pt} 
    
    \begin{adjustbox}{width=\textwidth}
    \begin{tabular}{lccccccc}
    \toprule
    \bf Metric / Dataset & rel-f1 & rel-avito & rel-event & rel-trial & rel-amazon & rel-stack & rel-hm \\
    \midrule
    Task & driver-dnf & user-clicks & user-repeat & study-outcome & item-churn & user-badge & user-churn \\
    \midrule
    w/o GNN branch & 0.7499 & 0.6751 & 0.8173 & 0.7094 & 0.8128 & 0.8917 & 0.6822 \\
    w/ GNN branch  & \bf 0.7608 & \bf 0.6844 & \bf 0.8357 & \bf 0.7254 & \bf 0.8297 & \bf 0.9044 & \bf 0.6997 \\
    \bottomrule
    \end{tabular}
    \end{adjustbox}
    \vspace{-10pt}
\end{table*}

\begin{table*}[t]
    \centering
     \caption{\textbf{More Component Analysis.} (a) \textbf{Selection Rule}: Top-$k$ selection via dot-product scores outperforms random-$k$ and bottom-$k$ on the same candidate nodes. (b) \textbf{Traversal Strategy}: BFS consistently surpasses DFS under matched budgets and settings. (c) \textbf{Other Temporal biases}: Comparison of AUC across different temporal bias formulations. (d) \textbf{Temporal Shuffling}: w/ temporal bias (GelGT) vs. w/o temporal bias (RelGT) under Shuffled Timestamps $\uparrow$: higher is better.}
     \vspace{-10pt}
    \label{tab:design_choices}
    \setlength{\tabcolsep}{1pt}

    \begin{subtable}[t]{0.25\textwidth}
        \centering\scriptsize
        \renewcommand{\arraystretch}{1.4}
        \begin{tabular}{lccc}
        \toprule
        \textbf{Task} & \textbf{Top-$k$} & \textbf{Rand-$k$} & \textbf{Bot-$k$} \\
        \midrule
        top3    & \textbf{0.8408} & 0.8311 & 0.8287 \\
        dnf     & \textbf{0.7608} & 0.7491 & 0.7384 \\
        repeat  & \textbf{0.8357} & 0.8180 & 0.8156 \\
        ignore  & \textbf{0.8779} & 0.8602 & 0.8564 \\
        outcome & \textbf{0.7254} & 0.7111 & 0.7207 \\
        \bottomrule
        \end{tabular}
        \caption{\small Selection Rule}
        \label{tab:topk_random_bottom}
    \end{subtable}\hfill
    \begin{subtable}[t]{0.25\textwidth}
        \centering\scriptsize
        \renewcommand{\arraystretch}{1.4}
        \begin{tabular}{lcc}
        \toprule
        \textbf{Task} & \textbf{BFS} & \textbf{DFS} \\
        \midrule
        top3    & \textbf{0.8408} & 0.8311 \\
        dnf     & \textbf{0.7608} & 0.7539 \\
        repeat  & \textbf{0.8357} & 0.8290 \\
        ignore  & \textbf{0.8779} & 0.8605 \\
        outcome & \textbf{0.7254} & 0.7102 \\
        \bottomrule
        \end{tabular}
        \caption{\small Traversal Strategy}
        \label{tab:bfs_dfs}
    \end{subtable}\hfill
    \begin{subtable}[t]{0.25\textwidth}
        \centering\scriptsize
        \renewcommand{\arraystretch}{1.05}
        \begin{tabular}{lccc}
        \toprule
        \textbf{Task} & \textbf{Linear} & \textbf{Embed.} & \textbf{Ours} \\
        \midrule
        dnf      & 0.75 & 0.75 & \textbf{0.76} \\
        clicks   & 0.67 & 0.67 & \textbf{0.68} \\
        repeat   & 0.75 & 0.74 & \textbf{0.76} \\
        outcome  & 0.70 & 0.71 & \textbf{0.72} \\
        churn    & 0.69 & 0.69 & \textbf{0.70} \\
        badge    & 0.89 & 0.88 & \textbf{0.90} \\
        hm-churn & 0.68 & 0.67 & \textbf{0.69} \\
        \bottomrule
        \end{tabular}
        \caption{\small More Temporal Biases}
        \label{tab:gaussian_compare}
    \end{subtable}\hfill
    \begin{subtable}[t]{0.25\textwidth}
        \centering\scriptsize
        \renewcommand{\arraystretch}{1.4}
        \begin{tabular}{lcc}
        \toprule
        \textbf{Task} & \textbf{RelGT} & \textbf{GelGT} \\
        \midrule
        top3    & 0.7812 & \textbf{0.8025} \\
        dnf     & 0.7177 & \textbf{0.7332} \\
        repeat  & 0.7201 & \textbf{0.7998} \\
        ignore  & 0.7568 & \textbf{0.8513} \\
        outcome & 0.6605 & \textbf{0.6798} \\
        \bottomrule
        \end{tabular}
        \caption{\small Temporal Shuffling}
        \label{tab:timestamp_shuffle}
    \end{subtable}
    \vspace{-20pt}
\end{table*}

\subsection{Configuration Analysis (\textbf{RQ3})}
\begin{wrapfigure}{r}{0.4\textwidth}
\vspace{-10pt}
    \centering
        \includegraphics[width=\linewidth, height=2.5cm]{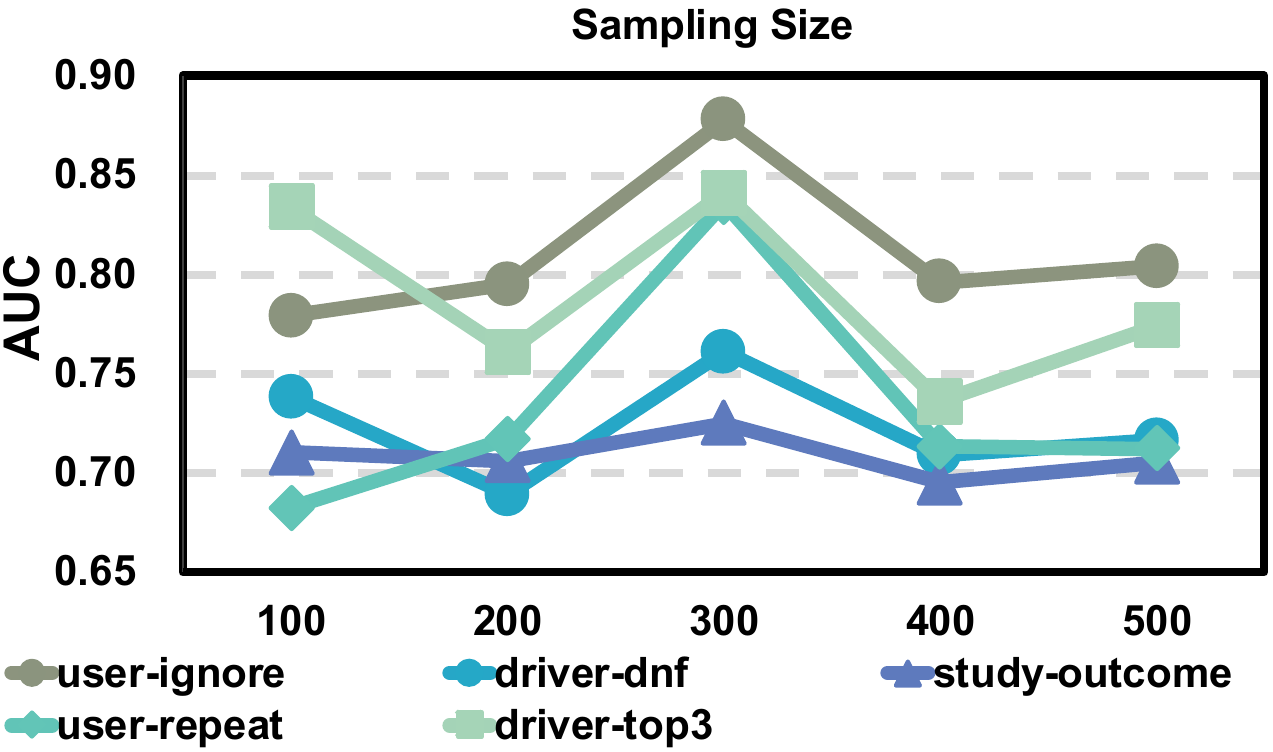}
        \vspace{-20pt}
        \caption{\small \bf Sampling Size}
        \label{fig:sample_size}
        \centering
            \includegraphics[width=\linewidth, height=2.5cm]{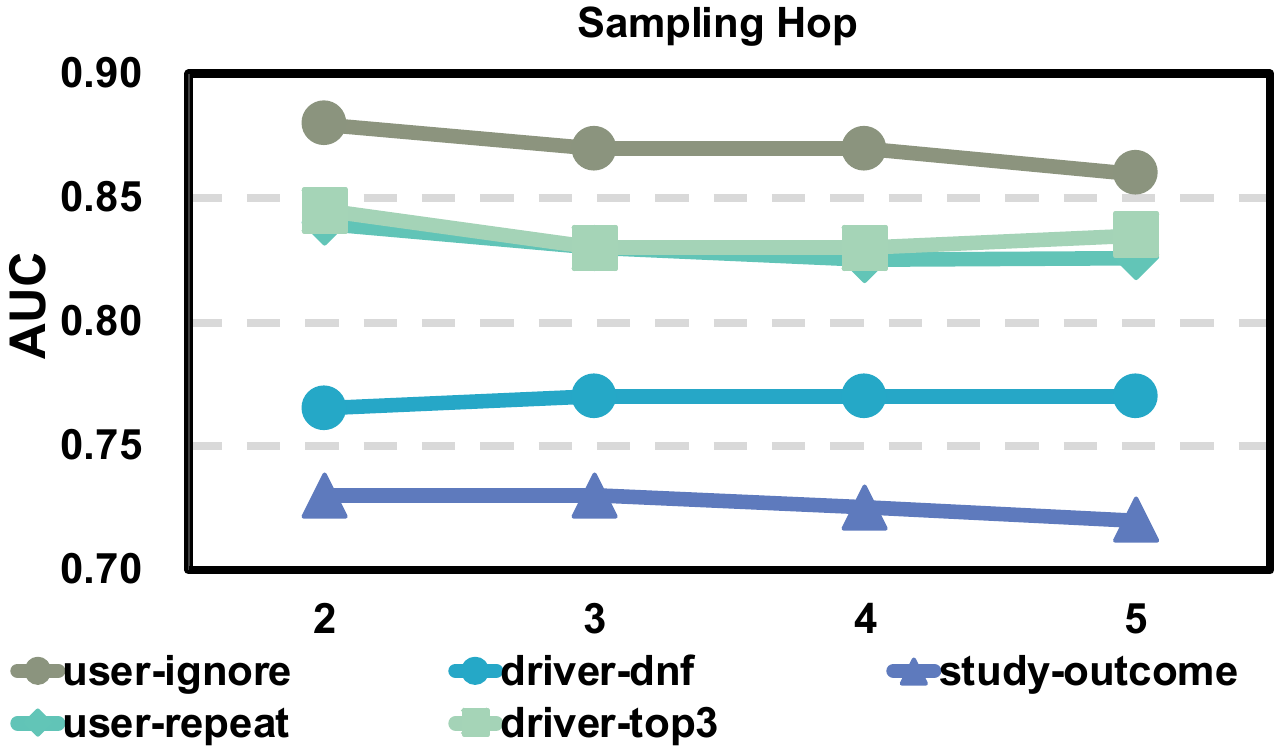}
        \vspace{-20pt}
        \caption{\small \bf Sampling Hops}
        \label{fig:sample_hop}
        \vspace{-20pt}
\end{wrapfigure}
\textbf{Sampling Size Analysis}.
 We vary the number of sampled nodes (i.e., sampling size) within $\{100, 200, 300, 400, 500\}$, as shown in Figure \ref{fig:sample_size}. The performance across all tasks reaches its peak at a sampling size of 300. Specifically, when the sampling size is less than 300, the model suffers from an information bottleneck and fails to capture sufficient semantic information. Conversely, when the sampling size exceeds 300, 
the limited number of truly relevant neighbors forces the similarity-based mechanism to incorporate semantically irrelevant nodes, which obscures the contribution of semantically relevant nodes. Therefore, we set the sampling size to 300, as it ensures sufficient information aggregation while avoiding interference from semantically irrelevant nodes. In the Appendix \ref{sampling budget}, we further analyze the sampling budget.

\textbf{Sampling Hop Analysis.}
\label{sec:hop}   
Then, we vary the sampling hops from $\{2, 3, 4, 5\}$, as summarized in Figure \ref{fig:sample_hop}. The performance remains consistent across all tasks with minimal fluctuations, indicating that GelGT is insensitive to the sampling hops. 
In GelGT, long-range information is incorporated through the semantic relevance-based selection of distant nodes in Stage 2, rather than through indiscriminate topological expansion.
This stability suggests that increasing the search depth beyond a moderate range provides limited benefit once semantic refinement is applied. Stage 1 exposes candidate nodes through topological expansion, while Stage 2 selectively retains distant nodes according to semantic relevance. Consequently, deeper expansion mainly increases the candidate pool and computational cost, rather than proportionally increasing useful information.
Details on the rationale behind the semantic refinement design are provided in Appendix \ref{refinement desigh}.

\subsection{Efficiency Analysis (\textbf{RQ4})}
\vspace{-5pt}
\begin{wrapfigure}{r}{0.6\textwidth}
    \centering
            \includegraphics[width=\linewidth]{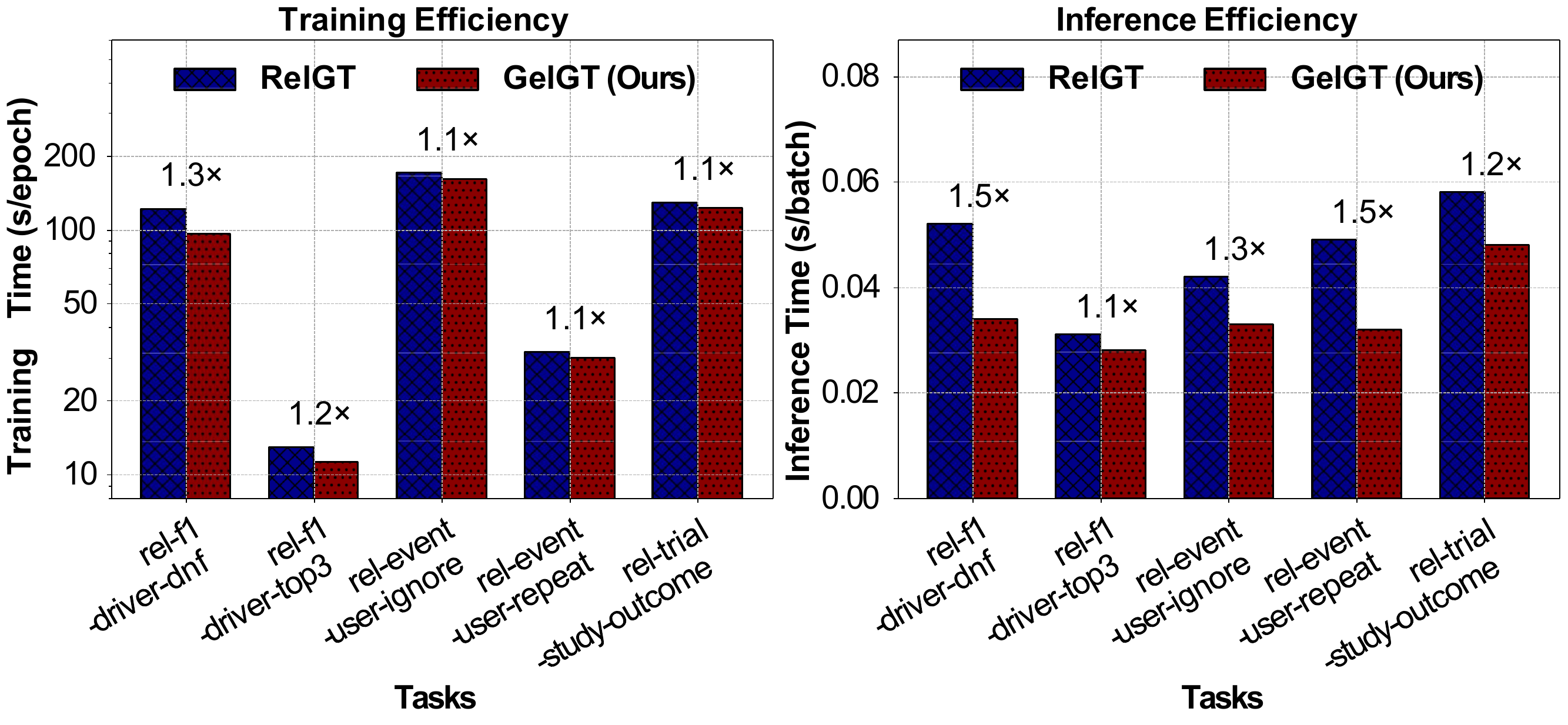}
             \vspace{-22pt}
        \caption{\small \textbf{Efficiency Comparison between RelGT and GelGT}}
        \label{fig:efficiency}
        \vspace{-10pt}
\end{wrapfigure}
We evaluate the efficiency of GelGT across five tasks by comparing its training time per epoch and inference time against state-of-the-art relational graph transformer models. As shown in Figure \ref{fig:efficiency}, GelGT consistently demonstrates lower time costs for both training and inference across all five evaluated tasks compared to baseline methods. This uniform reduction in computational overhead confirms that GelGT is more efficient for relational graph learning tasks.

\begin{wraptable}{r}{0.5\textwidth}
\vspace{-10pt}
    \centering
        \caption{\small \textbf{Comparison of Per-Node Sampling Time between Random Sampling and GelGT}}
        \vspace{-10pt}
        \label{tab:sampling_efficiency_}
        \begin{adjustbox}{max width=\linewidth}
        \begin{tabular}{lcc}
            \toprule
            \bf Task & \bf Random Sampling & \bf Ours \\
            \midrule
            rel-trial-study-outcome & 14ms & 14ms \\
            rel-ratebeer-beer\_rating-total\_score & 87ms & 88ms \\
            \bottomrule
        \end{tabular}
        \end{adjustbox}
        \centering
       \caption{\small \textbf{Scalability Comparison between Our Proposed Attention and FlashAttention \citep{DaoFERR22} across varying Sampled Node Sizes.}}
       \vspace{-10pt}
        \label{tab:attention_scalability_}
        \begin{adjustbox}{width=\linewidth}
        \begin{tabular}{lcccc}
            \toprule
            \multirow{2}{*}{\bf Context Size} & \multicolumn{2}{c}{\bf Ours} & \multicolumn{2}{c}{\bf FlashAttention} \\
            \cmidrule(lr){2-3} \cmidrule(lr){4-5}
            & \bf Epoch Time & \bf GPU Mem & \bf Epoch Time & \bf GPU Mem \\
            \midrule
            Size 300   & 225ms  & 2.2GB  & 210ms  & 1.9GB  \\
            Size 3000  & 1376ms & 70.0GB & 1004ms & 53.9GB \\
            Size 30000 & \textemdash\ OOM \textemdash & \textemdash\ OOM \textemdash & \textemdash\ OOM \textemdash & \textemdash\ OOM \textemdash \\
            \bottomrule
        \end{tabular}
        \end{adjustbox}
        \vspace{-10pt}
\end{wraptable}
This advantage does not come at the cost of attention computation efficiency: at the 300-node context we use, our attention costs only 1.07$\times$ the epoch time of FlashAttention (Table \ref{tab:attention_scalability_}). Furthermore, we evaluate our sampling efficiency in comparison with naive Random Sampling. Empirical results from Table \ref{tab:sampling_efficiency_} demonstrate that combining Breadth-First Search (BFS) with semantic similarity-based refinement introduces negligible latency, maintaining a computational cost highly comparable to Random Sampling.
Appendix~\ref{efficiency} further analyzes sampling latency, attention cost, and memory consumption under increasing context sizes.

\vspace{-5pt}
\section{Conclusion}
\vspace{-5pt}
We present GelGT, a Gaussian Relational Graph Transformer that improves long-range relational modeling by determining both which information to retain and how it should contribute to prediction. GelGT combines structure-semantic collaborative sampling to preserve relevant relational information with a Gaussian temporal bias to model its temporal focus and range. Experiments on 7 real-world relational datasets covering 21 prediction tasks demonstrate consistent improvements over existing relational graph learning methods.

\bibliography{iclr2027_conference}
\bibliographystyle{iclr2027_conference}

\newpage
\appendix
\tableofcontents
\newpage

\section{Proofs of Theorems.}
\label{sec:approofoftheorem}
\subsection{Upper Bound of Relative Structural Loss}
\label{lab:a.1}
\emph{Proof.} To characterize the multi-hop topological structure of nodes in the graph, we employ Katz centrality \citep{katz1953new} as our analytical tool. As a classic walk-based topological metric, Katz centrality systematically captures the structural reach of a node by computing a weighted sum of all walks of varying lengths associated with it. 
Formally, let $A$ be the adjacency matrix of the graph, and $\lambda$ be the attenuation factor determining the weight decay of paths. Let $\mathbf{C}^{\mathcal{G}}$ denote the vector of Katz centrality for all nodes in the full graph, and let $\mathbf{1}$ be the vector of all ones. The matrix-vector form of Katz centrality is given by the Neumann series:
\begin{equation}
    \mathbf{C}^{\mathcal{G}} = \sum_{k=1}^{\infty} (\lambda A)^k \mathbf{1}.
\end{equation}

The structural integrity sampling strategy preserves all walks of length $k=1$ and $k=2$. The structural loss vector $\boldsymbol{\Delta C}$, which consists of all walks of length $k \geq 3$, can be expressed as:
\begin{equation}
    \boldsymbol{\Delta C} = \sum_{k=3}^{\infty} (\lambda A)^k \mathbf{1}.
\end{equation}

We observe a critical recursive relationship between the error vector and the global centrality vector. By factoring out $(\lambda A)^2$ from the error series:
\begin{equation}
    \boldsymbol{\Delta C} = (\lambda A)^2 \left( \sum_{k=1}^{\infty} (\lambda A)^k \mathbf{1} \right) = (\lambda A)^2 \mathbf{C}^{\mathcal{G}}.
\end{equation}

To quantify the relative magnitude of the loss, we employ the matrix 2-norm (spectral norm). Using the consistency property of matrix norms ($\|M\mathbf{x}\| \leq \|M\| \|\mathbf{x}\|$), we have:
\begin{equation}
    \| \boldsymbol{\Delta C} \|_2 = \| (\lambda A)^2 \mathbf{C}^{\mathcal{G}} \|_2 \leq \| (\lambda A)^2 \|_2 \| \mathbf{C}^{\mathcal{G}} \|_2.
\end{equation}

Using the property that the spectral norm of a symmetric matrix is its spectral radius, i.e., $\|A\|_2 = \rho(A)$, and considering the definition $\lambda = c / \rho(A)$:
\begin{equation}
    \| (\lambda A)^2 \|_2 \leq (\lambda \|A\|_2)^2 = (\lambda \rho(A))^2 = \left( \frac{c}{\rho(A)} \cdot \rho(A) \right)^2 = c^2.
\end{equation}

Substituting this bound back into the inequality yields:
\begin{equation}
    \| \boldsymbol{\Delta C} \|_2 \leq c^2 \| \mathbf{C}^{\mathcal{G}} \|_2.
\end{equation}

Rearranging the terms, we obtain the strict upper bound for the relative structural loss:
\begin{equation}
    \frac{\| \boldsymbol{\Delta C} \|_2}{\| \mathbf{C}^{\mathcal{G}} \|_2} \leq c^2.
\end{equation}

Following the analysis of the local regime by Benzi and Klymko \citep{benzi2013limiting}, the scaling factor $c$ falls within the range of $\leq 0.1$. 

Substituting this value into Theorem \ref{lab:them1}, the relative structural loss is upper-bounded by:
\begin{equation}
    \frac{\| \boldsymbol{\Delta C} \|_2}{\| \mathbf{C}^{\mathcal{G}} \|_2} \leq (0.1)^2 = 1\%.
\end{equation}
This confirms that under standard operating conditions, the structural integrity sampling preserves at least 99\% of the effective structural information.

\subsection{Exponential Decay of Structural Sensitivity}
\emph{proof.}
\label{lab:lemma3}
To justify the differentiated pruning strategy in Stage 2, we analyze the sensitivity of the seed node's structural information to the removal of a specific neighbor at distance $k$. We define the structural sensitivity $\Delta \mathcal{C}(v, k)$ as the magnitude of change in the Katz centrality of $v$ when a node $u$ at distance $k$ (i.e., $\text{dist}(v, u) = k$) is removed from the graph.

Based on the definition of Katz centrality, the contribution of a node $u$ to $v$ is the weighted sum of all walks between them:
\begin{equation}
    \text{Contr}(u \to v) = \sum_{\gamma: u \rightsquigarrow v} \beta^{|\gamma|},
\end{equation}
where $\gamma$ represents a walk from $u$ to $v$, $\beta \in (0, 1)$ is the attenuation factor determining the weight decay of paths and $|\gamma|$ is its length.

When a neighbor $u$ at distance $k$ is pruned, all walks passing through $u$ are eliminated. The dominant term in this summation corresponds to the shortest path (geodesic) of length $k$. Higher-order paths (length $>k$) also contribute, but their weights are dampened by higher powers of $\beta$. Therefore, the structural loss is dominated by the leading term:
\begin{equation}
    \Delta \mathcal{C}_{k-hop}(v) \approx \beta^k \cdot N_{\text{paths}}(k) + \mathcal{O}(\beta^{k+1}),
\end{equation}
where $N_{\text{paths}}(k)$ denotes the number of shortest paths of length $k$ between $v$ and $u$ (typically small for local neighborhoods).

Now, we compare the sensitivity at $k=1$ (direct neighbors) and $k=2$ (two-hop neighbors):
\begin{align}
    \Delta \mathcal{C}_{1-hop}(v) &= \beta \cdot 1 + \mathcal{O}(\beta^2) \sim \mathcal{O}(\beta), \\
    \Delta \mathcal{C}_{2-hop}(v) &= \beta^2 \cdot N_{\text{paths}}(2) + \mathcal{O}(\beta^3) \sim \mathcal{O}(\beta^2).
\end{align}
Note that for $k=1$, there is exactly one direct edge, so the coefficient is 1.

The ratio of structural sensitivity between pruning a 2-hop neighbor and a 1-hop neighbor is:
\begin{equation}
    \frac{\Delta \mathcal{C}_{2-hop}(v)}{\Delta \mathcal{C}_{1-hop}(v)} \approx \frac{\mathcal{O}(\beta^2)}{\mathcal{O}(\beta)} = \mathcal{O}(\beta).
\end{equation}

Since the attenuation factor $\beta$ is typically chosen to be small (satisfying $\beta < 1/\rho(A)$), this ratio indicates that the structural perturbation caused by removing a 2-hop neighbor is an order of magnitude smaller than that of removing a 1-hop neighbor. This theoretical result confirms that 1-hop neighbors form the non-negligible structural skeleton ($\mathcal{O}(\beta)$ sensitivity), whereas 2-hop neighbors allow for aggressive semantic pruning with bounded structural risk ($\mathcal{O}(\beta^2)$ sensitivity). \qed

\subsection{Enhancement of Semantic Relevance-to-Noise Ratio}
\label{a.3}
As we prove in Appendix \ref{lab:lemma3}, pruning 2-hop neighbors has a much smaller impact on structural information compared to 1-hop neighbors. Therefore, we apply the following semantic refinement exclusively to 2-hop neighbors.

\emph{Proof.} The features of the sampled neighbors are aggregated to obtain a representation for the seed node, which is then used for downstream tasks. We define a general aggregation formula as follows: 
\begin{equation}
    \mathbf{h}_{agg} = \sum_{u \in \mathcal{N}(v)} w(v, u) \mathbf{h}_u
\end{equation}
In this formulation, $v$ represents the seed node, $u$ is a neighbor of $v$, $\mathbf{h}_u$ denotes the feature of neighbor $u$, $w(v,u)$ is the aggregation weight, and $\mathbf{h}_{\text{agg}}$ corresponds to the aggregated feature.

We can decompose the feature of neighbor $u$, $\mathbf{h}_u$, into a part that is semantically relevant to the seed node and a part that corresponds to irrelevant noise:
\begin{equation}
    \mathbf{h}_u = \mu_{vu}\mathbf{h}_v + \mathbf{h}_u^{\text{noise}}
\end{equation}

The first term corresponds to the part of $\mathbf{h}_u$ that is semantically relevant to the seed node $v$, with $\mu_{vu}$ representing the degree of relevance. The second term represents the part unrelated to $v$, i.e., random noise (e.g., heterophily), assumed to have variance $\sigma_{\text{noise}}^2$.

By substituting $\mathbf{h}_u$ into the aggregation formula, we can decompose the aggregated feature into two parts:
\begin{equation}
\mathbf{h}_{\text{agg}} = \sum_{u \in \mathcal{N}(v)} w(v, u) \mathbf{h}_u
= \underbrace{\sum_{u \in \mathcal{N}(v)} w(v, u) \mu_{vu} \mathbf{h}_v}_{\text{semantically relevant}} 
+ \underbrace{\sum_{u \in \mathcal{N}(v)} w(v, u) \mathbf{h}_u^{\text{noise}}}_{\text{noise}}
\end{equation}

Following the theoretical analysis in PairNorm \citep{zhaopairnorm}, which establishes that the information content of node features is strictly proportional to their squared norm, we measure the magnitude of the aggregated information using the $L_2$-norm. Therefore, the aggregated semantically relevant information is quantified as:
\begin{equation}
    \mathcal{I}_\text{relevance} = \left\| \left( \sum_{u \in \mathcal{N}(v)} w(v,u) \, \mu_{vu} \right) \mathbf{h}_v \right\|^2 = \left( \sum_{u \in \mathcal{N}(v)} w(v,u) \mu_{uv} \right)^2 \cdot \|\mathbf{h}_v\|^2
\end{equation}

The variance directly measures the destructive information or uncertainty \citep{10.5555/3295222.3295309} introduced by irrelevant neighbors. Therefore, we compute the semantically noisy information as:
\begin{equation}
    \mathcal{I}_{\text{noise}} = \sum_{u \in \mathcal{N}(v)} \mathrm{Var}(w(v,u) \mathbf{h}_u^{\text{noise}})
\end{equation}
Using the properties of variance, we can factor out the weight $w(v, u)$ to obtain:
\begin{equation}
    \mathcal{I}_{\text{noise}} = \sum_{u \in \mathcal{N}(v)} w(v,u)^2 \, \mathrm{Var}(\mathbf{h}_u^{\text{noise}})
\end{equation}
If the variance of each individual noise is $\sigma_{\text{noise}}^2$:
\begin{equation}
    \mathcal{I}_{\text{noise}} = \sum_{u \in \mathcal{N}(v)} w(v,u)^2\sigma_{\text{noise}}^2 = \sigma_{\text{noise}}^2\sum_{u \in \mathcal{N}(v)} w(v,u)^2
\end{equation}
Therefore, the ratio of relevant information to noise information is:
\begin{equation}
    \text{SNR} = \frac{\mathcal{I}_\text{relevance}}{\mathcal{I}_{noise}} = \frac{\left( \sum_{u \in \mathcal{N}(v)} w(v,u) \mu_{uv} \right)^2 \cdot \|\mathbf{h}_v\|^2}{\sigma_{\text{noise}}^2\sum_{u \in \mathcal{N}(v)} w(v,u)^2} = \underbrace{\frac{\|\mathbf{h}_v\|^2}{\sigma_{\text{noise}}^2}}_{\text{constant}} \cdot \frac{\left( \sum_{u \in \mathcal{N}(v)} w(v,u) \mu_{uv} \right)^2 }{\sum_{u \in \mathcal{N}(v)} w(v,u)^2}
\end{equation}
To demonstrate the effectiveness of our semantic refinement, we denote the refined 2-hop neighbors by $\mathcal{N}_{refine}(v)$ and the remaining subsets is denoted as $\mathcal{N}_{remaining}(v)$. 

Let $A$ and $B$ denote the relevance term and noise term in  SNR of our refinement strategy, respectively:
\begin{equation}
    A = \sum_{u \in \mathcal{N}_{refine}(v)} w(v,u) \mu_{vu}, \quad B = \sum_{u \in \mathcal{N}_{refine}(v)} w(v,u)^2.
\end{equation}

Similarly, for the remaining subset $\mathcal{N}_{remaining}(v)$, we define the relevance term $a$ and noise term $b$:
\begin{equation}
    a = \sum_{u \in {\mathcal{N}_{remaining}(v)}} w(v,u) \mu_{vu}, \quad b = \sum_{u \in \mathcal{N}_{remaining}(v)} w(v,u)^2
\end{equation}

 Nodes in $\mathcal{N}_{remaining}(v)$ have negligible correlation with the seed node, implying $\mu_{vu} \approx 0$ and thus $a \approx 0$. However, since these nodes objectively exist in the graph, they contribute a strictly positive noise term, i.e., $b > 0$.

Now, we compare the SNR of our refinement strategy ($\text{SNR}_{post}$) against the full 2-hop neighborhood strategy ($\text{SNR}_{pre}$):
\begin{equation}
    \text{SNR}_{after} = \frac{\mathcal{I}^{after}_{\text{relevance}}}{\mathcal{I}^{after}_{\text{noise}}} = \frac{\|\mathbf{h}_v\|^2}{\sigma_{\text{noise}}^2} \cdot \frac{A^2}{B}
\end{equation}
\begin{equation}
    \text{SNR}_{before} = \frac{\mathcal{I}^{before}_{\text{relevance}}}{\mathcal{I}^{before}_{\text{noise}}} = \frac{\|\mathbf{h}_v\|^2}{\sigma_{\text{noise}}^2} \cdot \frac{(A + a)^2}{B + b} \approx \frac{\|\mathbf{h}_v\|^2}{\sigma_{\text{noise}}^2} \cdot \frac{A^2}{B + b}.
\end{equation}

Since $b$ is positive and $B$ is positive, the denominator of the full 2-hop neighborhood strategy is strictly larger while the numerator remains approximately unchanged. This leads to the inequality:
\begin{equation}
    \frac{A^2}{B + b} < \frac{A^2}{B} \implies \text{SNR}_{before} < \text{SNR}_{after} \implies \frac{\mathcal{I}^{before}_{\text{relevance}}}{\mathcal{I}^{before}_{\text{noise}}} < \frac{\mathcal{I}^{after}_{\text{relevance}}}{\mathcal{I}^{after}_{\text{noise}}}.
\end{equation}

This inequality mathematically guarantees that pruning low-similarity neighbors effectively removes the noise denominator $b$ without sacrificing the relevance term numerator $A$, thereby strictly maximizing the ratio of relevant information to noise information of the aggregated representation. Consequently, this result rigorously confirms that our semantic refinement strategy functions as an effective denoising filter, eliminating irrelevant nodes to enhance the quality of node representation. \qed

\subsection{Convergence of Temporal Center}
\label{lab:a.4}

In this subsection, we provide a rigorous theoretical analysis demonstrating that the proposed attention mechanism is able to focus more effectively on temporally relevant nodes compared to existing attention mechanisms.

For the model to attend to relevant historical information, the learnable temporal center $\mu$ must converge to the ground-truth relevant timestamp $t^*$.

\emph{proof.} Consider a simplified objective where the model aims to maximize the alignment score $\mathcal{S}$ between a query and a ground-truth relevant event at time $t^*$. The optimization follows the gradient ascent direction $\nabla_\mu \mathcal{S}$. 

The baseline attention logit is derived from the dot product of sinusoidal encodings:
\begin{equation}
    \mathcal{S}_{\text{base}}(t^*) \approx \sum_{j=1}^{d/2} \cos(\omega_j (t^* - \mu)).
\end{equation}
The gradient with respect to $\mu$ is:
\begin{equation}
    \frac{\partial \mathcal{S}_{\text{base}}}{\partial \mu} = \sum_{j=1}^{d/2} \omega_j \sin(\omega_j (t^* - \mu)).
\end{equation}
This gradient function is a superposition of high-frequency sine waves. Its sign flips repeatedly as the distance $|t^* - \mu|$ changes. If the initialization of $\mu$ is far from $t^*$, the gradient direction is determined by local periodic behavior rather than the global distance to $t^*$. Consequently, the optimization is prone to getting trapped in local optima (false peaks) nearest to the initialization.

Our method incorporates an adaptive bias term $\mathcal{B}(t^*) = \exp\left( - \frac{(t^* - \mu)^2}{2\sigma^2} \right)$. The gradient contribution from this term is:
\begin{equation}
    \frac{\partial \mathcal{B}}{\partial \mu} = \mathcal{B}(t^*) \cdot \frac{t^* - \mu}{\sigma^2}.
\end{equation}
Let $\Delta = t^* - \mu$ be the alignment error. The sign of the gradient is determined by $\text{sgn}(\Delta)$:
\begin{itemize}
    \item If $\mu < t^*$ ($\Delta > 0$), the gradient is positive, pushing $\mu$ towards $t^*$.
    \item If $\mu > t^*$ ($\Delta < 0$), the gradient is negative, pushing $\mu$ towards $t^*$.
\end{itemize}
Unlike the baseline, the term $(t^* - \mu)$ acts as a linear \textbf{restoring force}. It provides a consistent global gradient direction that guides $\mu$ towards $t^*$, smoothing the optimization landscape and significantly expanding the basin of attraction for the true timestamp.

\subsection{Attention Score Ratio Enhancement}
\label{lab:a.5}

Building upon the analysis in Appendix \ref{lab:a.4}, which guarantees that the learnable temporal center $\mu$ converges to the ground-truth relevant timestamp $t^*$, we now prove that the proposed attention mechanism effectively concentrates attention scores on these relevant nodes. Specifically, we demonstrate that our method enforces a strict probabilistic preference for events temporally close to $t^*$, suppressing distant noise that the baseline model fails to filter.

\emph{Proof.} 
Let $t^*$ be the ground-truth relevant timestamp. Based on the gradient guidance property proved in Appendix \ref{lab:a.4}, we assume the model has successfully learned the optimal temporal reference, i.e., $\mu \approx t^*$.

Consider a relevant event $e_{\text{rel}}$ and an irrelevant noise event $e_{\text{noise}}$. We define their respective temporal differences relative to the ground-truth timestamp $t^*$ as:
\begin{equation}
    \delta_{\text{relevant}} = t_{\text{relevant}} - t^*,
\end{equation}
and
\begin{equation}
    \delta_{\text{noise}} = t_{\text{noise}} - t^*.
\end{equation}
The relevant event is temporally proximal to the target $t^*$, while the noise event is distant. This structural difference is formalized by the inequality of their magnitudes:
\begin{equation}
    |\delta_{\text{relevant}}| \ll |\delta_{\text{noise}}|.
\end{equation}

The baseline attention logit is $\mathcal{S}_{\text{base}}(\delta) = \sum \cos(\omega_j \delta)$.
According to Bohr's theorem on almost periodic functions~\citep{bohr2018almost}, $\mathcal{S}_{\text{base}}$ does not vanish at infinity. Specifically:
\begin{equation}
\limsup_{|\delta|\to\infty} \mathcal{S}_{\text{base}}(\delta) = \sup_{\delta\in\mathbb{R}} \mathcal{S}_{\text{base}}(\delta).
\end{equation}
This implies that there exist infinitely many distant timestamps $t_{noise}$ (where $|\delta_{noise}| \to \infty$) such that:
\begin{equation}
\mathcal{S}_{\text{base}}(t_{noise}) \approx \mathcal{S}_{\text{base}}(t_{relevant}).
\end{equation}
Consequently, the baseline model cannot structurally guarantee that a relevant node receives a higher score than a distant noise node. The discriminative margin $\Delta_{\text{base}} = \mathcal{S}_{\text{base}}(t_{relevant}) - \mathcal{S}_{\text{base}}(t_{noise})$ has a lower bound of approximately $0$ (or even negative values due to oscillations), indicating a failure to distinguish relevance from periodic noise.

In our method, with $\mu \approx t^*$, the Gaussian bias becomes $\mathcal{B}(\delta) = \exp( - \delta^2 / 2\sigma^2 )$. The total logit is $\mathcal{S}_{\text{ours}}(\delta) = \mathcal{S}_{\text{base}}(\delta) + \mathcal{B}(\delta)$.

The discriminative margin between the relevant event and the noise event is:
\begin{equation}
\Delta_{\text{ours}} = \mathcal{S}_{\text{ours}}(t_{relevant}) - \mathcal{S}_{\text{ours}}(t_{noise}) = \Delta_{\text{base}} + \left[ \mathcal{B}(\delta_{relevant}) - \mathcal{B}(\delta_{noise}) \right].
\end{equation}

Since $|\delta_{relevant}| \ll |\delta_{noise}|$, and the Gaussian function is strictly monotonically decreasing with respect to distance from the center, we have:
\begin{equation}
\mathcal{B}(\delta_{relevant}) \gg \mathcal{B}(\delta_{noise}) \approx 0.
\end{equation}
Even in the worst-case scenario for the baseline (where $\Delta_{\text{base}} \le 0$ due to periodic misalignment), the Gaussian term introduces a strictly positive margin enhancement.

Specifically, considering the asymptotic limit where the noise becomes infinitely distant ($|\delta_{noise}| \to \infty$) and the relevant event is perfectly aligned with the learned center ($\delta_{rel} \to 0$), we have:
\begin{equation}
\lim_{|\delta_{noise}| \to \infty} \mathcal{B}(\delta_{noise}) = 0 \quad \text{and} \quad \lim_{\delta_{relevant} \to 0} \mathcal{B}(\delta_{relevant}) = 1.
\end{equation}
This imposes a strict exponential penalty on the noise in the score ratio:
\begin{equation}
 \frac{\alpha_{noise}^\text{w/ \ G}}{\alpha_{relevant}^\text{w/ \ G}} = \exp\left( \mathcal{S}_{\text{base}}({noise}) - \mathcal{S}_{\text{base}}({relevant}) - 1 \right) = e^{-1} \cdot \frac{\alpha_{noise}^\text{w/o \ G}}{\alpha_{relevant}^\text{w/o \ G}}.
\end{equation}
where $\alpha_{noise}$ and $\alpha_{relevant}$ denote the attention scores for temporal noise and related nodes, respectively; $\mathrm{\mathbf{e}}$ is Euler's number; and w/ G'' and w/o G'' distinguish the attention mechanisms with and without Gaussian bias.

Combining the result that $\mu$ converges to $t^*$ (Appendix \ref{lab:a.4}) with the Gaussian decay property, we prove that our method effectively assigns significantly higher attention scores to strictly temporally relevant nodes. Unlike the baseline, which is susceptible to distant "periodic ghosts," our mechanism imposes a structural envelope that filters out irrelevant temporal information, ensuring that the model's focus remains robustly on the learned temporal context.

\newpage
\section{Experimental Details.}
\label{sec:apexpdetailes}
\subsection{Training Setting}
GelGT is trained in an end-to-end manner without any pre-training stage to ensure a fair comparison with baseline methods such as RelGT. We do not adopt pre-training in our setting for two reasons: (i) it introduces significant computational overhead, and (ii) it often fails to transfer effectively across relational graphs (e.g., Griffin still underperforms even after fine-tuning). In contrast, direct supervised training allows the model to better adapt to the target relational structure, leading to more stable and competitive performance.

\subsection{Benchmark Details.}
In this section, we present a comprehensive overview of the datasets and evaluation tasks sourced from RelBench \citep{0001RHHHDFLYZHL24}. 
RelBench encompasses seven datasets spanning diverse relational database domains, such as e-commerce, clinical records, social networks, and sports. 
These datasets are meticulously curated from their respective domains, with training set sizes ranging from 1.3K to 5.4M records, totaling 47M training instances. 
Each dataset supports various predictive tasks, including modeling user engagement with advertisements within a four-day window or forecasting the primary outcomes of clinical trials over a year. 

Collectively, RelBench defines 30 tasks categorized into entity classification, entity regression, and recommendation. 
Our evaluation focuses on 21 classification and regression tasks, as GelGT is primarily designed as a node representation learning framework for Relational Deep Learning (RDL). 
Following pror work \citep{dwivedi2026relational}, we exclude recommendation tasks in this study due to their specific requirements, such as target node identification \citep{YouGYL21} or the adoption of pair-wise learning architectures \citep{YuanZHNHSSHLLF25}. 
Detailed statistics for all evaluated datasets and tasks are summarized in Table \ref{tab:statistics}.
\begin{table}[h!]
  \centering
  \caption{Dataset and task statistics from RelBench \citep{0001RHHHDFLYZHL24} used for our evaluation.}
  \label{tab:statistics}
  \renewcommand{\arraystretch}{1.1}
  \setlength{\tabcolsep}{4pt}
 \resizebox{\linewidth}{!}{
  \begin{tabular}{lllccclll}
    \toprule
    \multirow{2}{*}{\textbf{Dataset}} & \multirow{2}{*}{\textbf{Task name}} & \multirow{2}{*}{\textbf{Task type}} & \multicolumn{3}{c}{\textbf{\#Rows of training table}} & \textbf{\#Unique} & \textbf{\%train/test} & \textbf{\#Dst} \\
    & & & \textbf{Train} & \textbf{Validation} & \textbf{Test} & \textbf{Entities} & \textbf{Entity Overlap} & \textbf{Entities} \\
    \midrule
    \multirow{7}{*}{rel-amazon} 
      & \texttt{user-churn} & entity-cls & 4,732,555 & 409,792 & 351,885 & 1,585,983 & 88.0 & --- \\
      & \texttt{item-churn} & entity-cls & 2,559,264 & 177,689 & 166,842 & 416,352 & 93.1 & --- \\
      & \texttt{user-ltv} & entity-reg & 4,732,555 & 409,792 & 351,885 & 1,585,983 & 88.0 & --- \\
      & \texttt{item-ltv} & entity-reg & 2,707,679 & 166,978 & 178,334 & 427,537 & 93.5 & --- \\
      & \texttt{user-item-purchase} & recommendation & 5,112,803 & 351,876 & 393,985 & 1,632,909 & 87.4 & 12,562,384 \\
      & \texttt{user-item-rate} & recommendation & 3,667,157 & 257,939 & 292,609 & 1,481,360 & 81.0 & 7,665,611 \\
      & \texttt{user-item-review} & recommendation & 2,324,177 & 116,970 & 127,021 & 894,136 & 74.1 & 5,406,835 \\
    \midrule
    \multirow{4}{*}{rel-avito} 
      & \texttt{ad-ctr} & entity-reg & 5,100 & 1,766 & 1,816 & 4,997 & 59.8 & --- \\
      & \texttt{user-clicks} & entity-cls & 59,454 & 21,183 & 47,996 & 66,449 & 45.3 & --- \\
      & \texttt{user-visits} & entity-cls & 86,619 & 29,979 & 36,129 & 63,405 & 64.6 & --- \\
      & \texttt{user-ad-visit} & recommendation & 86,616 & 29,979 & 36,129 & 63,402 & 64.6 & 3,616,174 \\
    \midrule
    \multirow{3}{*}{rel-event} 
      & \texttt{user-attendance} & entity-reg & 19,261 & 2,014 & 2,006 & 9,694 & 14.6 & --- \\
      & \texttt{user-repeat} & entity-cls & 3,842 & 268 & 246 & 1,514 & 11.5 & --- \\
      & \texttt{user-ignore} & entity-cls & 19,239 & 4,185 & 4,010 & 9,799 & 21.1 & --- \\
    \midrule
    \multirow{3}{*}{rel-f1} 
      & \texttt{driver-dnf} & entity-cls & 11,411 & 566 & 702 & 821 & 50.0 & --- \\
      & \texttt{driver-top3} & entity-cls & 1,353 & 588 & 726 & 134 & 50.0 & --- \\
      & \texttt{driver-position} & entity-reg & 7,453 & 499 & 760 & 826 & 44.6 & --- \\
    \midrule
    \multirow{3}{*}{rel-hm} 
      & \texttt{user-churn} & entity-cls & 3,871,410 & 76,556 & 74,575 & 1,002,984 & 89.7 & --- \\
      & \texttt{item-sales} & entity-reg & 5,488,184 & 105,542 & 105,542 & 105,542 & 100.0 & --- \\
      & \texttt{user-item-purchase} & recommendation & 3,878,451 & 74,575 & 67,144 & 1,004,046 & 89.2 & 13,428,473 \\
    \midrule
    \multirow{5}{*}{rel-stack} 
      & \texttt{user-engagement} & entity-cls & 1,360,850 & 85,838 & 88,137 & 88,137 & 97.4 & --- \\
      & \texttt{user-badge} & entity-cls & 3,386,276 & 247,398 & 255,360 & 255,360 & 96.9 & --- \\
      & \texttt{post-votes} & entity-reg & 2,453,921 & 156,216 & 160,903 & 160,903 & 97.1 & --- \\
      & \texttt{user-post-comment} & recommendation & 21,239 & 825 & 758 & 11,453 & 59.9 & 44,940 \\
      & \texttt{post-post-related} & recommendation & 5,855 & 226 & 258 & 5,924 & 8.5 & 7,456 \\
    \midrule
    \multirow{5}{*}{rel-trial} 
      & \texttt{study-outcome} & entity-cls & 11,994 & 960 & 825 & 13,779 & 0.0 & --- \\
      & \texttt{study-adverse} & entity-reg & 43,335 & 3,596 & 3,098 & 50,029 & 0.0 & --- \\
      & \texttt{site-success} & entity-reg & 151,407 & 19,740 & 22,617 & 129,542 & 42.0 & --- \\
      & \texttt{condition-sponsor-run} & recommendation & 36,934 & 2,081 & 2,057 & 3,956 & 98.4 & 533,624 \\
      & \texttt{site-sponsor-run} & recommendation & 669,310 & 37,003 & 27,428 & 445,513 & 48.3 & 1,565,463 \\
    \bottomrule
  \end{tabular}
  }
\end{table}
\subsection{Datasets}
In this section, we include the details on the datasets in RelBench \citep{0001RHHHDFLYZHL24} which we use for our evaluation.

\begin{itemize}[left=-1pt]
    \item \texttt{rel-amazon.} The Amazon E-commerce dataset consists of product details, user information, and review interactions from Amazon's platform, including metadata like pricing and categories, along with review ratings and content.
    \item \texttt{rel-avito.} Avito's marketplace dataset contains search queries, advertisement characteristics, and contextual information from this major online trading platform that facilitates transactions across various categories including real estate and vehicles.
    
    \item \texttt{rel-event.} The Event Recommendation dataset from Hangtime mobile app tracks users' social planning, capturing interactions, event details, demographic data, and social connections to reveal how relationships impact user behavior.
    
    \item \texttt{rel-f1.} The F1 dataset provides comprehensive Formula 1 racing information since 1950, documenting drivers, constructors, manufacturers, and circuits with detailed records of race results, standings, and specific data on various racing sessions and pit stops.
    
    \item \texttt{rel-hm.} H\&M's dataset contains customer-product interactions from their e-commerce platform, featuring customer demographics, product descriptions, and purchase histories.
    
    \item \texttt{rel-stack.} The Stack Exchange dataset documents activity from this network of Q\&A websites, including user biographies, posts, comments, edits, votes, and question relationships where users earn reputation through contributions.
    
    \item \texttt{rel-trial.} The clinical trial dataset from the AACT initiative has study protocols and outcomes, containing trial designs, participant information, intervention details, and results metrics, serving as a key resource for medical research.
\end{itemize}

{\paragraph{Temporal information in RelBench.}
We clarify that the temporal information follows the standard RelBench \citep{0001RHHHDFLYZHL24} setting. We address four aspects explicitly.

\textbf{What temporal information means.} Each RelBench prediction sample has a seed time, which indicates when the prediction is made. Historical events or interactions may also have their own timestamps. GelGT uses the temporal relation between a historical neighbor's timestamp and the seed time.

\textbf{Whether it is provided by the dataset.} Yes. RelBench task tables provide Entity ID, seed time, and target label. For example, \texttt{rel-event/user-repeat} contains \texttt{user} and \texttt{timestamp}, \texttt{rel-avito/user-clicks} contains \texttt{UserID} and \texttt{timestamp}, and \texttt{rel-f1/driver-dnf} contains \texttt{driverId} and \texttt{date}.

\textbf{Whether all nodes have timestamps.} No. Only time-aware event or interaction tables have time columns, such as \texttt{events.start\_time}, \texttt{event\_attendees.start\_time}, \texttt{VisitStream.ViewDate}, \texttt{results.date}, and \texttt{studies.start\_date}. Static entity tables such as \texttt{UserInfo} may not have timestamps.

\textbf{Single or multiple timestamps.} For a single event or interaction record, the temporal information is usually one timestamp. An entity can be connected to multiple historical records, so it may be associated with multiple timestamps through its interactions.

\textbf{How GelGT uses it.} The seed time filters future records and prevents temporal leakage. For the retained historical neighbors, GelGT computes their time interval to the seed time. The Gaussian temporal bias then uses this interval to model whether a historical neighbor is temporally relevant to the current prediction.
\par}

\subsection{Tasks}
The following entity classification and regression tasks are defined in RelBench for the above datasets.

\begin{enumerate}[left=-1pt]
    \item \textbf{rel-amazon}
    \begin{enumerate}
        \item \texttt{user-churn}: Predict whether a user will discontinue reviewing products within the next three months.
        \item \texttt{item-churn}: Predict if a product will have no reviews in the next three months.
        \item \texttt{user-ltv}: Estimate the total monetary value of merchandise in dollars that a user will purchase and review within the next three months.
        \item \texttt{item-ltv}: Estimate the total monetary value of purchases and reviews a product will receive during the next three months.
    \end{enumerate}

    \item \textbf{rel-avito}
    \begin{enumerate}
        \item \texttt{user-visits}: Predict if a user will engage with several (advertisements) ads within the upcoming four days.
        \item \texttt{user-clicks}: Predict whether a user will interact with multiple ads through clicking within the upcoming four days.
        \item \texttt{ad-ctr}: Estimate the interaction probability for an ad, assuming it receives an interaction within four days.
    \end{enumerate}

    \item \textbf{rel-event}
    \begin{enumerate}
        \item \texttt{user-attendance}: Estimate the number of events a user will confirm attendance to (RSVP yes or maybe) within the upcoming seven days.
        \item \texttt{user-repeat}: Predict whether a user will join an event (RSVP yes or maybe) within the upcoming seven days, provided they attended an event during the previous fourteen days.
        \item \texttt{user-ignore}: Predict whether a user will disregard or ignore more than two events invitations within the upcoming seven days.
    \end{enumerate}

    \item \textbf{rel-f1}
    \begin{enumerate}
        \item \texttt{driver-dnf}: Predict if a driver will not finish a race within the upcoming month.
        \item \texttt{driver-top3}: Determine if a driver will achieve a top-three qualifying position in a race within the upcoming month.
        \item \texttt{driver-position}: Estimate a driver's average finishing placement across all races in the upcoming two months.
    \end{enumerate}

    \item \textbf{rel-hm}
    \begin{enumerate}
        \item \texttt{user-churn}: Predict whether a customer will not perform any transactions in the upcoming week.
        \item \texttt{item-sales}: Estimate total revenue generated by a product in the upcoming week.
    \end{enumerate}

    \item \textbf{rel-stack}
    \begin{enumerate}
        \item \texttt{user-engagement}: Predict whether a user will contribute through voting, posting, or commenting within the upcoming three months.
        \item \texttt{user-badge}: Predict whether a user will secure a new badge within the upcoming three months.
        \item \texttt{post-votes}: Estimate the number of votes a user's post will accumulate over the upcoming three months.
    \end{enumerate}

    \item \textbf{rel-trial}
    \begin{enumerate}
        \item \texttt{study-outcome}: Predict whether a clinical trial will achieve its principal outcome within the upcoming year.
        \item \texttt{study-adverse}: Estimate the number of patients who will experience significant adverse effects or mortality in a clinical trial over the upcoming year.
        \item \texttt{site-success}: Estimate the success rate of a clinical trial site in the upcoming year.
    \end{enumerate}
\end{enumerate}

\subsection{Hyperparameter Settings}
\label{setup}
\paragraph{Sampling Module.} In our experimental setup, to accommodate the varying demands for neighborhood information density across different downstream tasks, we employ a two-stage neighborhood sampling strategy. 

We implement differentiated configurations for sampling sizes based on task characteristics. For the node classification task, we initially sample $300$ candidate neighbors in the first stage, followed by a filtering mechanism that retains $200$ nodes as input. Conversely, for the node regression task---to address potentially more complex feature dependencies---we increase the initial sampling budget to $500$ nodes, ultimately selecting $300$ nodes for model training. This tailored configuration aims to strike an optimal balance between computational efficiency and model performance. 

\begin{table*}[t!]
\small
    \centering
    \caption{Detailed hyperparameter settings and implementation specifications.}
    \label{tab:full_hyperparams}
    \renewcommand{\arraystretch}{1.2}
    \begin{tabular}{l|c|l}
        \toprule
        \textbf{Hyperparameter} & \textbf{Value} & \textbf{Description / Note} \\
        \midrule
        \multicolumn{3}{c}{\textit{\textbf{Model Architecture}}} \\
        \midrule
        Hidden Dimension ($d$) & $512$ & Embedding size for all encoders \\
        Global Layers ($L_{global}$) & $4$ & Number of model layers (from script) \\
        Local GNN Depth & $3$ & Layers within internal GNNStack \\
        Attention Heads & $4$ & Multi-head attention mechanism \\
        Feed-forward Network (FFN) Ratio & $2$ & Hidden dim expansion in FFN \\
        Positional Encoding Dim & $128$ & Learnable structural PE \\
        Number of Centroids & $4096$ & For vector quantization (VQ) \\
        VQ Decay & $0.99$ & EMA decay rate for codebook update \\
        \midrule
        \multicolumn{3}{c}{\textit{\textbf{Optimization \& Training}}} \\
        \midrule
        Optimizer & Adam & Standard optimization algorithm \\
        Learning Rate & $1 \times 10^{-4}$ & Base learning rate \\
        Weight Decay & $1 \times 10^{-5}$ & L2 regularization coefficient \\
        Batch Size & $64$ & Adjusted for memory constraints \\
        Dropout Rate & $0.3$ & Applied to FFN and Attention \\
        Max Training Steps & $500$ & Number of steps per epoch \\
        Total Epochs & $10$ & Total training duration \\
        Warmup Steps & $10$ & Linear warmup strategy \\
        \midrule
        \multicolumn{3}{c}{\textit{\textbf{Data Sampling \& Processing}}} \\
        \midrule
        Neighborhood Scope & $2$-hop & Maximum hop distance sampled \\
        \textbf{Classification Configuration} & $300 \to 200$ & Sample 300, select top-200 \\
        \textbf{Regression Configuration} & $500 \to 300$ & Sample 500, select top-300 \\
        \bottomrule
    \end{tabular}
\end{table*}
\paragraph{Model Module.} Our proposed model adopts a unified hybrid architecture designed to capture both local structural topology and global temporal-semantic dependencies. 

To effectively handle the heterogeneous and temporal nature of the relational graph, we employ a set of domain-specific encoders \citep{dwivedi2026relational} to project diverse input features into a unified latent space $\mathbb{R}^{d}$. Specifically, we utilize learnable embedding layers to encode node type heterogeneity and relative hop distances from the central node. Furthermore, continuous timestamps are processed via a \texttt{TimeEncoder} to preserve temporal dynamics, while raw tabular node features are encoded using a \texttt{TabularEncoder} based on feature statistics. To enhance subgraph structural awareness, a learnable positional encoding is also incorporated. All resulting embeddings are normalized via LayerNorm, concatenated, and projected to a dimension of $d=512$ through a 2-layer MLP mixture network.

The core processing layer, \texttt{GelGTLayer}, processes the refined subgraph sequence $X_{in}$ through two parallel branches to extract complementary features. The GNN Branch is designed to explicitly capture local neighborhood topology using a 3-layer GraphSAGE \citep{HamiltonYL17} backbone. Each layer applies message passing followed by LayerNorm \citep{BaKH16}, GELU \citep{hendrycks2016gaussian} activation, and Dropout ($p=0.1$). Crucially, residual connections \citep{HeZRS16} are incorporated at each layer to facilitate gradient flow and prevent over-smoothing \citep{ChenWHDL20, 0001RGBWR21, WuSZFYW19, LiHW18}, yielding a structure-aware representation $H_{gnn}$. 

Parallel to the GNN, the Attention Branch employs a Transformer-based \texttt{LocalModule} to capture long-range dependencies and temporal interactions. By utilizing Multi-Head Self-Attention (Heads=$4$) integrated with temporal embeddings, this branch allows the model to dynamically weigh neighbors based on both semantic similarity and temporal proximity, producing a context-aware representation $H_{attn}$.

To optimally combine the structural and semantic representations, we introduce a learnable gating parameter $\eta$. The final representation $H_{final}$ is computed as a dynamic weighted sum:
\begin{equation}
    H_{final} = \eta \cdot H_{attn} + (1 - \eta) \cdot H_{gnn}
\end{equation}
where $\sigma(\cdot)$ denotes the Sigmoid function. This mechanism allows the model to adaptively balance the contribution of GNN and Transformer modules based on the specific requirements of the downstream task. Finally, a 2-layer MLP head generates the prediction.

The detailed hyperparameter configurations and implementation specifications are summarized in Table~\ref{tab:full_hyperparams}.

\subsection{Hardcore Configurations}
\label{computation env}
We conduct all experiments with:
\begin{itemize}[left=-1pt]
    \item Operating System: Ubuntu 20.04.6 LTS.
    \item CPU: Intel(R) Xeon(R) Platinum 8358 CPU @ 2.60GHz.
    \item GPU: NVIDIA Tesla A100 SMX4 with 80GB of Memory.
    \item Software: CUDA 12.1, Python 3.10.18, PyTorch \citep{PaszkeGMLBCKLGA19} 2.5.1.
\end{itemize}

\newpage
\section{Additional Baselines and Datasets}
\label{sec:apadditionalbaseline}
To further validate the robustness, generalizability, and architectural superiority of the Gaussian Relational Graph Transformer (GelGT), we expand our empirical evaluation to include the highly challenging SALT dataset \citep{gu2026relbenchv2largescalebenchmark} and benchmark against a broader spectrum of emerging and advanced baseline models, including Rel-LLM \citep{WuDL25}, Nodeformer \citep{WuZLWY22}, RelGNN \citep{ChenKL25}, and RGP \citep{lachi2025integrating}.

The SALT dataset presents a significantly more complex environment for relational deep learning. To ensure a direct and rigorous comparison with existing temporal neighbor sampling methods such as RGP, we evaluated GelGT across multiple classification tasks within the SALT dataset. We report the Mean Reciprocal Rank (MRR) to align with standard evaluation metrics used in recent literature.

As demonstrated in Table \ref{tab:salt_dataset}, GelGT consistently achieves state-of-the-art performance across all tasks. The substantial performance gap confirms that our approach---specifically the Gaussian temporal bias combined with structure-aware sampling---captures long-range relational dependencies more effectively than traditional heterogeneous methods (HGT) and recent temporal sampling strategies (RGP).

We further isolate the specific advantages of our structural-semantic-temporal collaborative modeling by comparing GelGT against two advanced architectures: Rel-LLM and Nodeformer.

We evaluated our model against Rel-LLM, a recent method leveraging large language models for relational graphs, and Nodeformer, an architecture specifically optimized for long-range dependencies. As shown in Table \ref{tab:emerging_baselines} (evaluated via AUC), GelGT consistently outperforms both baselines. This confirms that GelGT's success is not merely a byproduct of general model capacity, but fundamentally stems from our decoupled continuous Gaussian temporal bias, which precisely focuses on temporal patterns without distorting semantic embeddings.

{\subsection{Comparison with Recent Graph Transformers and Relational Foundation Models}
\label{ap:recent_baselines}
 In addition to the baselines reported above, we compare GelGT with Relational Transformer (RT)~\citep{ranjan2026relational}, a recent relational foundation model.

GelGT outperforms RT on all five tasks, with the largest relative gain of 32.9\% on study-outcome. The relative gains are 13.2\% on item-churn, 6.7\% on user-churn, 12.4\% on user-clicks, 6.9\% on user-visits, and 32.9\% on study-outcome. This directly shows that GelGT is not a RelBench-specific tuned tool: it remains stronger than a recent relational foundation model under the same downstream evaluation.

\begin{table}[h]
\centering
\caption{Comparison with Relational Transformer (RT), a recent relational foundation model. $\uparrow$: higher is better.}
\label{tab:rt}
\begin{tabular}{lcc}
\toprule
\textbf{Task} & \textbf{GelGT} & \textbf{RT} \\
\midrule
item-churn     & \textbf{0.8297} & 0.7330 \\
user-churn     & \textbf{0.7050} & 0.6610 \\
user-clicks    & \textbf{0.6844} & 0.6090 \\
user-visits    & \textbf{0.6695} & 0.6260 \\
study-outcome  & \textbf{0.7254} & 0.5460 \\
\bottomrule
\end{tabular}
\end{table}
\par}

\newpage
\section{Additional Efficiency Experiments}
\label{efficiency}
In this section, we provide a detailed analysis of the computational efficiency of the proposed GelGT architecture. We specifically evaluate the overhead introduced by our structure-semantic collaborative sampling strategy and the Gaussian temporal bias-based attention mechanism, comparing them against standard baselines and optimized kernels.
\subsection{Sampling Efficiency}
A potential concern regarding our proposed sampling mechanism is whether the Breadth-First Search (BFS) combined with semantic similarity-based refinement introduces significant computational latency compared to naive Random Sampling. However, empirical results demonstrate that our approach maintains a cost highly comparable to Random Sampling.

The efficiency of our sampling strategy stems from its dual-stage design:
\begin{itemize}
    \item {\bf Stage 1 (Structural Integrity Sampling):} Under our standard configuration (e.g., a 300-node limit), the overhead of managing the BFS queue is marginal. It remains on the same scale as the random number generation and deduplication operations required by Random Sampling.
    \item {\bf Stage 2 (Semantic Refinement):} The similarity calculation and selection process are executed in parallel for all subgraphs within a batch, ensuring minimal additional latency.
\end{itemize}
To validate this, we measured the single-node sampling times across two RelBench \citep{gu2026relbenchv2largescalebenchmark} datasets, including rel-ratebeer, a massive-scale graph with tens of millions of nodes. As shown in Table \ref{tab:sampling_efficiency}, our method remains just as efficient as Random Sampling, proving its feasibility for large-scale relational databases.
\begin{table}[htb]
    \centering
    \caption{Comparison of per-node sampling time between random sampling and the proposed method.}
    \label{tab:sampling_efficiency}
    \begin{tabular}{lcc}
        \toprule
        \bf Task & \bf Random Sampling & \bf Ours \\
        \midrule
        rel-trial-study-outcome & 14ms & 14ms \\
        rel-ratebeer-beer\_rating-total\_score & 87ms & 88ms \\
        \bottomrule
    \end{tabular}
\end{table}

\subsection{Attention Efficiency}
We also analyze the computational cost of our Gaussian Temporal Bias-based Attention. While modern optimized attention kernels like FlashAttention significantly accelerate vanilla attention, they are often less compatible with arbitrary pairwise biases. We demonstrate that our method remains highly competitive through two key perspectives: practical scalability and context size requirements.

We compared the training runtime (Epoch Time) and GPU memory consumption of our attention mechanism against FlashAttention across different context sizes (Table \ref{tab:attention_scalability}). For our standard context size of 300 nodes, FlashAttention yields only a marginal 1.07$\times$ speedup, as the small context does not trigger typical memory bottlenecks. When aggressively scaling the sample size up to 3,000, our method's runtime remains in the same order of magnitude as FlashAttention. At an extreme scale of 30,000 nodes, both methods inevitably face Out-Of-Memory (OOM) errors due to hardware limitations. This confirms that the practical training overhead of our attention mechanism scales similarly to optimized vanilla attention.

A common assumption is that larger graphs inherently require substantially more context (e.g., thousands of nodes). However, our structure-semantic collaborative sampling effectively filters out irrelevant noise, eliminating the need for massive context windows. As demonstrated on the rel-ratebeer regression task (a dataset with over 13.7 million nodes), GelGT successfully captures the necessary structural information and achieves SOTA performance (Table \ref{tab:ratebeer_results}) while strictly maintaining an extremely small context size of 300 nodes.

By improving sampling quality and strictly controlling the sample size, GelGT fundamentally avoids the quadratic computational bottlenecks of attention mechanisms, rendering the lack of FlashAttention compatibility a non-issue in practical deployments.
\begin{table*}[htb]
    \centering
    
    \begin{minipage}{0.48\textwidth}
        \centering
        \caption{Scalability comparison between our proposed attention and FlashAttention across varying sampled node sizes.}
        \label{tab:attention_scalability}
        \begin{adjustbox}{width=\textwidth}
        \begin{tabular}{lcccc}
            \toprule
            \multirow{2}{*}{\bf Context Size} & \multicolumn{2}{c}{\bf Ours} & \multicolumn{2}{c}{\bf FlashAttention} \\
            \cmidrule(lr){2-3} \cmidrule(lr){4-5}
            & \bf Epoch Time & \bf GPU Mem & \bf Epoch Time & \bf GPU Mem \\
            \midrule
            Size 300   & 225ms  & 2.2GB  & 210ms  & 1.9GB  \\
            Size 3000  & 1376ms & 70.0GB & 1004ms & 53.9GB \\
            Size 30000 & \textemdash\ OOM \textemdash & \textemdash\ OOM \textemdash & \textemdash\ OOM \textemdash & \textemdash\ OOM \textemdash \\
            \bottomrule
        \end{tabular}
        \end{adjustbox}
    \end{minipage}
    \hfill

    \begin{minipage}{0.48\textwidth}
        \centering
        \caption{Performance comparison on the rel-ratebeer dataset (regression task). Best values (in terms of mean performance) are in \textbf{bold}. $\downarrow$ indicates lower is better.}
        \label{tab:ratebeer_results}
        \begin{adjustbox}{width=\textwidth}
        \begin{tabular}{lcccc}
            \toprule
            \bf Task $\downarrow$ & \bf RDL & \bf RelGT & \bf LightGBM & \bf Ours \\
            \midrule
            rel-beer\_ratings-total\_score & 0.36 & 0.33 & 0.47 & \textbf{0.32} \\
            \bottomrule
        \end{tabular}
        \end{adjustbox}
    \end{minipage}
\end{table*}

\subsection{Efficiency on Extreme Graph Structures}

Our sampling strategy is constrained by dual factors: a maximum BFS depth and a fixed candidate budget. This design ensures consistent efficiency across graph structures with varying densities.

In dense graphs, BFS quickly reaches the fixed budget (300 nodes), which bounds the cost of subsequent semantic filtering. In sparse graphs, the BFS traversal terminates earlier due to depth limitation, naturally reducing the candidate pool and further accelerating computation.

To validate this behavior, we construct extreme graph settings following the skewness-based generation protocol in~\citep{ding2024play}. We report the average sampling time in Table~\ref{tab:extreme_sampling}.

\begin{table}[h]
\centering
\caption{Sampling efficiency under extreme graph structures.}
\label{tab:extreme_sampling}
\begin{tabular}{ccc}
\hline
\textbf{Sparse Graph} & \textbf{trial-study-outcome} & \textbf{Dense Graph} \\
\hline
11 ms & 14 ms & 15 ms \\
\hline
\end{tabular}
\end{table}

The results indicate that our sampling strategy maintains stable efficiency under dense conditions and even improves slightly in sparse settings, demonstrating robustness across graph distributions.

\newpage
\section{Detailed Description of Encoder Modules}
\label{lab:encoder}
In this section, we describe the specific encoder modules designed to process the heterogeneous attributes associated with nodes in our relational graph framework. Following previous work \citep{dwivedi2026relational}, these encoders transform raw input features---including node types, hop distances, temporal information, tabular attributes, and structural positions---into dense vector representations suitable for downstream learning tasks.

\paragraph{Neighbor Node Type Encode.} The Neighbor Node Type Encoder maps discrete heterogeneity information into a continuous latent space using a learnable lookup table (Embedding Layer). It instantiates a parameter matrix $\mathbf{E}_{\text{type}} \in \mathbb{R}^{|\mathcal{T}| \times d}$, where $|\mathcal{T}|$ is the number of node types. The encoder takes integer indices as input and performs a direct embedding lookup to project categorical entity constraints into dense vector representations.

\paragraph{Neighbor Hop Encoder.} The Neighbor Hop Encoder introduces structural inductive bias by encoding the shortest-path distance using a dedicated Embedding Layer. To accommodate 0-based or 1-based indexing, the scalar hop counts are shifted and mapped via an embedding matrix $\mathbf{E}_{\text{hop}} \in \mathbb{R}^{K_{\text{max}} \times d}$. This non-linear mapping allows the model to learn distance-dependent feature transformations rather than treating distance merely as a scalar weight.

\paragraph{Neighbor Time Encoder.} The Neighbor Time Encoder transforms continuous scalar time intervals using a Sinusoidal Positional Encoding (PE) mechanism followed by a Linear Projection. Specifically, the module first projects scalar time differences $\Delta t$ into $\mathbb{R}^d$ using fixed Fourier feature frequencies (Sinusoidal PE). The output is then passed through a learnable Linear Layer (nn.Linear). To handle invalid temporal data (e.g., negative time deltas), the module utilizes a learnable Mask Parameter, combining the linear output and the mask vector via a gating mechanism. 

\paragraph{Neighbor Tabular Feature Encoder.} The Neighbor Tabular Feature Encoder processes heterogeneous tabular attributes (numerical, categorical, and timestamps) using a ResNet backbone adapted for tabular data via the torch frame library. For each node type, the encoder first applies specific column encoders---Linear Encoders for numerical features and Embedding Encoders for categorical features. These processed features are then fed into a deep ResNet (consisting of multiple residual blocks with Linear, BatchNorm, and ReLU layers) to extract high-order interactions between columns.

\paragraph{GNN Positional Encoder.} The GNN-based Positional Encoder captures local topological structures using a Graph Isomorphism Network (GIN) backbone. Node features are initialized either with random noise or Laplacian Eigenvector Positional Encodings (LapPE) projected via a Linear layer. The structure is processed by stacked GIN Convolutional Layers, where each layer utilizes a Multi-Layer Perceptron (MLP) as the aggregation function. This MLP is composed of Linear $\to$ BatchNorm $\to$ ReLU $\to$ Linear blocks. Residual connections are applied across layers, and the final embedding is obtained by aggregating layer-wise outputs via a generic pooling function (e.g., concatenation or max-pooling) followed by a final linear transformation.

{\paragraph{GNN Branch: Architecture, Output, and Fusion.}
The GNN branch uses a 3-layer GraphSAGE~\citep{HamiltonYL17} backbone. Each layer performs message passing followed by LayerNorm, GELU, and Dropout, and residual connections are incorporated at each layer, yielding a structure-aware representation $\mathbf{H}_{\mathrm{gnn}}$.\\
\textbf{Output.} The output of the GNN branch is $\mathbf{H}_{\mathrm{gnn}}$. It is not a single graph-level vector; it contains node-level structure-aware representations for all nodes in the refined sampled subgraph.\\
\textbf{Fusion inside the branch.} Since the branch uses GraphSAGE, the internal fusion is the standard GraphSAGE neighbor aggregation and message passing along PK--FK edges. In other words, the branch fuses local neighbor information through edge-based aggregation, and then uses multi-layer message passing with residual connections to produce $\mathbf{H}_{\mathrm{gnn}}$.\\
\textbf{Fusion across branches.} The outputs of the GNN branch and the Gaussian temporal bias-based attention branch are combined by a learnable gate, so that the model can interpolate between structural and attentional evidence rather than committing to either branch a priori. Table~\ref{tab:gnn_ablation} and Table~\ref{tab:branch_wise} show that neither branch alone matches the fused model.
\par}

\newpage
\section{Additional Analysis and Discussion}
\label{sec:apanalysis}
\subsection{Semantic Refinement Design}
\label{refinement desigh}
We describe the design of semantic refinement from two aspects: (1) its exclusion at 1-hop neighbors, and (2) its behavior under non-homophilic settings.

\textbf{Exclusion at 1-hop neighbors.}
Semantic refinement is not applied to 1-hop neighbors in order to preserve high-quality first-order signals. In relational databases, 1-hop neighbors correspond to direct PK-FK relations and provide the most reliable structural and semantic information for the target node. Filtering at this stage may remove essential information and degrade the quality of initial node representations. Preserving complete first-order neighborhoods is also consistent with prior findings (e.g., LightGCN~\citep{0001DWLZ020}), which emphasize the importance of retaining full local connectivity for effective representation learning.

\textbf{Behavior under non-homophilic settings.}
Semantic refinement is designed to operate beyond strict homophily assumptions. Instead of relying on node-type similarity, the method evaluates neighbors in a shared latent space \citep{WangJSWYCY19} constructed from tabular attributes and positional encodings. In this space, similarity reflects task-relevant semantic correlation. Consequently, heterogeneous neighbors are retained as long as they contribute useful information to the target task. The refinement mechanism mainly acts on higher-hop neighbors to suppress less relevant or noisy signals, while preserving informative dependencies even in heterophilic scenarios.

{\paragraph{Validation of the selection score.}
The dot product is applied to encoded node embeddings, not raw attributes, so it measures relevance in the learned task space, as in dot-product attention and dense retrieval~\citep{KarpukhinOMLWEC20}. From the same nodes, we compare top-($k$), random-($k$), and bottom-($k$) selection ranked by our dot-product score. Top-($k$) performs best on all five tasks, showing that the score ranks nodes by downstream usefulness rather than acting as an arbitrary heuristic.

\begin{table}[h]
\centering
\caption{Comparison of selection rules applied to the same candidate pool. Top-$k$ ranked by the dot-product score outperforms both random-$k$ and bottom-$k$ on all five tasks. $\uparrow$: higher is better.}
\label{tab:topk_random_bottom}
\begin{tabular}{lccc}
\toprule
\textbf{Task} & \textbf{Top-$k$} & \textbf{Random-$k$} & \textbf{Bottom-$k$} \\
\midrule
driver-top3    & \textbf{0.8408} & 0.8311 & 0.8287 \\
driver-dnf     & \textbf{0.7608} & 0.7491 & 0.7384 \\
user-repeat    & \textbf{0.8357} & 0.8180 & 0.8156 \\
user-ignore    & \textbf{0.8779} & 0.8602 & 0.8564 \\
study-outcome  & \textbf{0.7254} & 0.7111 & 0.7207 \\
\bottomrule
\end{tabular}
\end{table}
\par}

Overall, semantic refinement balances information preservation and noise reduction by maintaining complete first-order neighborhoods and selectively filtering higher-order neighbors based on task-relevant semantic signals.

\subsection{Generality of the Sampling Strategy}
\label{sampleing}
To demonstrate the generality of the proposed sampling strategy, we further evaluate it on another relational graph learning architecture, RelGNN. We integrate our sampling strategy into RelGNN without modifying its original model design, and compare the performance with its vanilla version.

The results are summarized in Table~\ref{tab:generality}. We observe that the proposed sampling strategy consistently improves the performance of RelGNN across multiple tasks. This suggests that the effectiveness of our method is not tied to a specific model architecture.

We attribute these gains to the ability of our semantic similarity-based filtering to reduce noisy or less informative neighbors, thereby providing cleaner and more relevant neighborhood signals for downstream relational learning.

\begin{table}[t]
\centering
\caption{Performance comparison (AUC) between RelGNN and RelGNN with the proposed sampling strategy.}
\label{tab:generality}
\begin{tabular}{lcc}
\toprule
\textbf{Dataset} & \textbf{RelGNN} & \textbf{RelGNN (w/ Ours)} \\
\midrule
rel-avito-user-visits     & 0.6517 & \bf 0.6673 \\
rel-event-user-repeat     & 0.7661 & \bf 0.7898 \\
rel-trial-study-outcome   & 0.7003 & \bf 0.7290 \\
rel-amazon-item-churn     & 0.8043 & \bf 0.8118 \\
rel-stack-user-engagement & 0.8975 & \bf 0.9088 \\
\bottomrule
\end{tabular}
\end{table}

{\paragraph{Traversal strategy.}
Under a fixed sampling budget, the traversal strategy determines which nodes enter the subgraph, and BFS better matches our goal. Our sampling aims to preserve local PK--FK relations. BFS expands layer by layer, first covering 1-hop neighbors and then higher-hop nodes. DFS may go deep along few paths and miss other direct 1-hop relations, weakening local coverage. This is consistent with node2vec~\citep{GroverL16}, where BFS-style sampling preserves nearby neighborhoods and supports structural-equivalence-oriented representations. Under the same hop number, budget, semantic refinement, and training setting, BFS outperforms DFS on all five tasks, with the largest relative gain of 2.1\% on study-outcome.

\begin{table}[h]
\centering
\caption{Effect of the graph traversal strategy. BFS consistently outperforms DFS under an identical sampling budget, hop number, and training setting. $\uparrow$: higher is better.}
\label{tab:bfs_dfs}
\begin{tabular}{lcc}
\toprule
\textbf{Task} & \textbf{BFS} & \textbf{DFS} \\
\midrule
driver-top3    & \textbf{0.8408} & 0.8311 \\
driver-dnf     & \textbf{0.7608} & 0.7539 \\
user-repeat    & \textbf{0.8357} & 0.8290 \\
user-ignore    & \textbf{0.8779} & 0.8605 \\
study-outcome  & \textbf{0.7254} & 0.7102 \\
\bottomrule
\end{tabular}
\end{table}

\paragraph{Generality without the GNN branch.}
Long-range dependency is mainly modeled by the Gaussian attention branch. We therefore repeat the evaluation inside an attention-only model, in which the GNN branch is removed and all remaining settings are unchanged. As shown in Table~\ref{tab:attn_only_ablation}, removing structural integrity sampling, semantic refinement, or adaptive Gaussian bias consistently reduces the attention-only model on all five tasks. These results show that the gains of GelGT are not tied to the GNN branch and are not obtained by simply increasing the receptive field; they come from enabling useful long-range dependency through structural preservation, semantic denoising, and temporal weighting.

\begin{table}[h]
\centering
\caption{Ablations inside the attention-only branch. The three modules remain beneficial even when the GNN branch is removed, showing that the long-range gain does not come from simply enlarging the receptive field. $\uparrow$: higher is better.}
\label{tab:attn_only_ablation}
\begin{tabular}{lcccc}
\toprule
\textbf{Task} & \textbf{Attn-only} & \textbf{w/o Struct.} & \textbf{w/o Sem.} & \textbf{w/o Gauss.} \\
\midrule
driver-top3    & \textbf{0.8327} & 0.8237 & 0.8193 & 0.8122 \\
driver-dnf     & \textbf{0.7499} & 0.7370 & 0.7201 & 0.7186 \\
user-repeat    & \textbf{0.8173} & 0.7998 & 0.7947 & 0.7951 \\
user-ignore    & \textbf{0.8501} & 0.8465 & 0.8384 & 0.8322 \\
study-outcome  & \textbf{0.7094} & 0.7024 & 0.6969 & 0.6978 \\
\bottomrule
\end{tabular}
\end{table}
\par}

\subsection{Homophily vs. Heterophily in Relational Graphs}
In this section, we investigate whether relational graphs derived from databases exhibit homophilic or heterophilic properties, and clarify the role of low-similarity neighbors in our learning framework.

\paragraph{Heterophily of relational graphs.}
As discussed in the main paper (Line 90) and established in prior studies~\citep{zhu2020beyond, 0001RHHHDFLYZHL24, dwivedi2026relational}, relational graphs constructed from real-world databases are typically heterogeneous and temporal in nature. To further quantify this property, we compute the Edge Homophily Ratio~\citep{zhu2020beyond} on all evaluated datasets.

As shown in Table~\ref{tab:homophily}, all datasets consistently exhibit low homophily ratios (all $< 0.5$), indicating strong heterophily. This empirical evidence supports the conclusion that relational graphs in our setting are inherently heterophilous rather than homophilous.

\begin{table}[t]
\centering
\caption{Edge homophily ratio across datasets.}
\label{tab:homophily}
\begin{tabular}{lc}
\toprule
\textbf{Task} & \textbf{Edge Homophily Ratio} \\
\midrule
rel-f1-driver-dnf          & 0.23 \\
rel-avito-user-clicks      & 0.17 \\
rel-event-user-repeat      & 0.13 \\
rel-trial-study-outcome    & 0.22 \\
rel-amazon-item-churn      & 0.21 \\
rel-stack-user-badge       & 0.21 \\
rel-hm-user-churn          & 0.25 \\
\bottomrule
\end{tabular}
\end{table}

\paragraph{Role of low-similarity neighbors.}
A key design choice in our method is to filter out low-similarity neighbors, which are treated as noisy or less informative connections. This design is supported by both empirical observations and the underlying modeling mechanism.

From an empirical perspective, as shown in Figure \ref{fig:sample_size} of the main paper, increasing the number of sampled neighbors---thereby introducing more low-similarity neighbors---leads to a consistent degradation in performance. This indicates that such neighbors tend to introduce noise rather than useful signal for prediction.

From a mechanistic perspective, following the formulation in~\citep{WangJSWYCY19}, our model projects tabular attributes and positional encodings into a shared semantic latent space, where similarity reflects task-relevant semantic correlation rather than structural homophily. Under this view, neighbors with low similarity in the learned space are more likely to be task-irrelevant and thus act as noise. Filtering them therefore improves representation quality and predictive performance.

Overall, these results suggest that low-similarity neighbors are generally not beneficial for the considered tasks under our semantic similarity space, and instead may negatively affect model performance when included indiscriminately.

\subsection{Analysis of Gaussian Temporal Bias}
In this section, we analyze the design choice of the proposed Gaussian temporal bias and compare it with alternative temporal modeling strategies to better understand its advantages.

\paragraph{Comparison with alternative temporal bias mechanisms.}
We first compare our method with two representative alternatives: (1) scaled linear temporal biases~\citep{presstrain}, and (2) learnable temporal embeddings~\citep{li2020time}.

Scaled linear biases impose a monotonically increasing penalty with respect to temporal distance, effectively favoring recent interactions. However, such a design is limited in modeling non-monotonic or periodic temporal dependencies commonly observed in real-world relational data. In contrast, our Gaussian formulation introduces a learnable mean parameter, which allows the attention peak to shift adaptively toward task-relevant historical periods rather than strictly prioritizing recent events.

Learnable temporal embeddings, on the other hand, discretize time into bins and assign independent parameters to each interval. While expressive, this approach increases parameter complexity linearly with the temporal range and may introduce redundancy. In comparison, our Gaussian bias only introduces two learnable scalars per attention head (mean and variance), making it significantly more parameter-efficient and scalable across varying time spans.

\paragraph{Empirical comparison.}
We further evaluate the effectiveness of different temporal bias designs across multiple tasks. As shown in Table~\ref{tab:gaussian_compare}, our Gaussian temporal bias consistently outperforms both baselines.

\begin{table}[t]
\centering
\caption{Performance comparison (AUC) of different temporal bias mechanisms.}
\label{tab:gaussian_compare}
\begin{tabular}{lccc}
\toprule
\textbf{Task} & \textbf{Scaled Linear Bias} & \textbf{Temporal Embedding} & \textbf{Ours} \\
\midrule
f1-driver-dnf         & 0.75 & 0.75 & 0.76 \\
avito-user-clicks     & 0.67 & 0.67 & 0.68 \\
event-user-repeat     & 0.75 & 0.74 & 0.76 \\
trial-study-outcome   & 0.70 & 0.71 & 0.72 \\
amazon-item-churn     & 0.69 & 0.69 & 0.70 \\
stack-user-badge      & 0.89 & 0.88 & 0.90 \\
hm-user-churn         & 0.68 & 0.67 & 0.69 \\
\bottomrule
\end{tabular}
\end{table}

{\paragraph{Range and robustness of the temporal bias.}
A natural concern is that the Gaussian kernel is bounded in $[0,1]$, so the injected bias might be too small to influence attention. This concern comes from a misunderstanding of how the bias is used in GelGT. The raw Gaussian kernel is indeed in $[0,1]$, but GelGT does not add this raw value directly. We use a learnable linear mapping,
\begin{equation}
    \mathrm{Bias}_{\mathrm{time}} = \mathrm{Linear}\!\left(\exp\!\left(-\frac{(\Delta t - \mu)^2}{\sigma^2}\right)\right),
\end{equation}
so the final temporal bias added to attention is not restricted to $[0,1]$; its scale is learned during training. Moreover, the bias is added before the softmax,
\begin{equation}
    \mathrm{Softmax}\!\left(\frac{\mathbf{Q}\mathbf{K}^\top}{\sqrt{d}} + \mathrm{Bias}_{\mathrm{time}}\right),
\end{equation}
so it is part of the attention logit rather than a post-processing term. Even a logit change of 1 can change the relative attention weight by a factor of $e$, so its effect is not inherently small. Consistently, removing the adaptive Gaussian bias degrades performance in Tables~\ref{tab:tab.3} and~\ref{tab:tab.4}, and Table~\ref{tab:gaussian_compare} shows that our Gaussian bias outperforms both a scaled linear bias and a temporal embedding.

To further verify that this gain comes from meaningful temporal relevance rather than from the mere presence of a bias term, we randomly shuffle timestamps in the test graph. Performance drops on all five tasks, with the largest relative drop of 6.3\% on study-outcome. This confirms that the benefit depends on the actual temporal ordering of neighbors, rather than on the bias term acting as an unconditional additive constant.

\begin{table}[h]
\centering
\caption{Effect of randomly shuffling timestamps in the test graph. Destroying temporal ordering degrades performance on all five tasks. $\uparrow$: higher is better.}
\label{tab:timestamp_shuffle}
\begin{tabular}{lcc}
\toprule
\textbf{Task} & \textbf{w/o Shuffle} & \textbf{w/ Shuffle} \\
\midrule
driver-top3    & \textbf{0.8408} & 0.8025 \\
driver-dnf     & \textbf{0.7608} & 0.7332 \\
user-repeat    & \textbf{0.8357} & 0.7998 \\
user-ignore    & \textbf{0.8779} & 0.8513 \\
study-outcome  & \textbf{0.7254} & 0.6798 \\
\bottomrule
\end{tabular}
\end{table}
\par}

\paragraph{Discussion.}
The results demonstrate that the Gaussian temporal bias effectively captures more flexible temporal patterns compared to monotonic or discretized formulations. This flexibility, together with its parameter efficiency, enables consistent improvements across diverse relational prediction tasks.

\subsection{Additional Discussion on Design Choices}
In this section, we further clarify the design motivations behind the proposed components and highlight their relevance to the unique challenges of relational deep learning (RDL).

\paragraph{Gaussian temporal modeling.}
Timestamped interactions in relational databases often exhibit continuous, irregular, and non-monotonic temporal dynamics. To effectively capture such patterns, we propose a Gaussian temporal bias mechanism, which introduces only two learnable scalar parameters per attention head. 

Compared to scaled linear biases~\citep{presstrain}, which enforce a strictly monotonic temporal decay, and learnable temporal embeddings~\citep{li2020time}, which discretize time and scale parameters with temporal granularity, our formulation provides a more flexible and compact representation of temporal influence. In particular, the learnable mean enables adaptive shifting of the attention peak toward task-relevant historical periods, while the variance controls the temporal sensitivity range. This design offers a principled and parameter-efficient alternative for continuous temporal modeling in relational graphs.

\paragraph{Temporal-aware BFS sampling.}
Our sampling strategy is also carefully designed for relational databases and is not a direct application of standard BFS. Instead, we incorporate explicit temporal constraints into the BFS traversal process. This design ensures that sampled subgraphs preserve both (i) complete primary--foreign key relational structures and (ii) strict temporal causality.

Unlike prior RDL sampling methods~\citep{zhu2020beyond, 0001RHHHDFLYZHL24, dwivedi2026relational}, which either ignore temporal ordering or fail to jointly preserve structural and temporal consistency, our approach enables the extraction of causally valid relational subgraphs. This makes it particularly suitable for learning over dynamic relational systems where both structure and time are critical.

{\paragraph{Interaction between the three components.}
The ablation numbers in Tables~\ref{tab:tab.3} and~\ref{tab:tab.4} are close because each variant removes only one component while keeping the other two, so all variants remain strong. The key point is the consistent direction: removing structural sampling, semantic refinement, or Gaussian temporal bias always reduces performance. To further isolate the relative contribution of each component, we additionally conduct a full combination ablation over all $2^3 = 8$ settings on driver-dnf.

\begin{table}[h]
\centering
\caption{Full $2^3$ combination ablation on driver-dnf. Structural sampling, semantic refinement, and adaptive Gaussian temporal bias each contribute positively, and the full model performs best. $\uparrow$: higher is better.}
\label{tab:full_combination}
\begin{tabular}{cccc}
\toprule
\textbf{Sampling} & \textbf{Semantic} & \textbf{Temporal} & \textbf{driver-dnf} \\
\midrule
w/  & w/  & w/  & \textbf{0.7608} \\
w/  & w/  & w/o & 0.7495 \\
w/  & w/o & w/  & 0.7422 \\
w/  & w/o & w/o & 0.7413 \\
w/o & w/  & w/  & 0.7486 \\
w/o & w/  & w/o & 0.7430 \\
w/o & w/o & w/  & 0.7398 \\
w/o & w/o & w/o & 0.7333 \\
\bottomrule
\end{tabular}
\end{table}

Compared with the corresponding one-module-removed variant, structural sampling brings 1.6\% relative gain, semantic refinement 2.5\%, and Gaussian temporal bias 1.5\%. The full model performs best, improving by 3.8\% over removing all three.

\paragraph{Contribution of the two branches.}
We further add the GNN-only result under the same encoder, sampling strategy, hidden size, training budget, and random seed. The comparison includes GNN-only, attention-only, and full GelGT. Compared with the full model, GNN-only drops by up to 3.9\%, and attention-only drops by up to 3.2\%, confirming that the two branches are complementary. The GNN branch captures local PK--FK structure, while attention models global and temporal interactions. The full model outperforming both single-branch variants shows that fusion is necessary.

\begin{table}[h]
\centering
\caption{Branch-wise comparison. GNN-only and attention-only are both inferior to the fused model, showing that the two branches are complementary. $\uparrow$: higher is better.}
\label{tab:branch_wise}
\begin{tabular}{lccc}
\toprule
\textbf{Task} & \textbf{Full} & \textbf{GNN-only} & \textbf{Attn-only} \\
\midrule
driver-top3    & \textbf{0.8408} & 0.8318 & 0.8327 \\
driver-dnf     & \textbf{0.7608} & 0.7311 & 0.7499 \\
user-repeat    & \textbf{0.8357} & 0.8044 & 0.8173 \\
user-ignore    & \textbf{0.8779} & 0.8588 & 0.8501 \\
study-outcome  & \textbf{0.7254} & 0.7113 & 0.7094 \\
\bottomrule
\end{tabular}
\end{table}
\par}

\paragraph{Summary.}
Overall, the proposed designs are specifically motivated by the characteristics of relational databases, where both temporal irregularity and schema-driven structure play central roles. These components jointly provide a unified and efficient framework for relational representation learning.

\subsection{Sampling Budget Analysis}
\label{sampling budget}

To justify the choice of a fixed candidate size in our original design, we analyze the effect of sampling budget on both performance and efficiency. To better understand the effective neighborhood size, we additionally introduce an auxiliary diagnostic mechanism with adaptive soft-gating, which is used solely for analysis and is not involved in training or inference of the proposed model. This analysis reveals that the effective number of retained neighbors consistently stabilizes around a fixed range, providing empirical support for our initial design choice.

As shown in Table~\ref{tab:sampling_budget}, the average number of retained effective neighbors across datasets ranges between 249 and 347, with a mean value close to 300. This indicates that a sampling budget around 300 is sufficient to capture informative relational semantics while avoiding excessive redundancy.

\begin{table}[h]
\centering
\caption{Average effective number of retained neighbors under adaptive soft-gating.}
\label{tab:sampling_budget}
\begin{tabular}{lc}
\hline
\textbf{Dataset} & \textbf{Avg. Effective Neighbors} \\
\hline
f1-driver-dnf & 298 \\
avito-user-clicks & 275 \\
event-user-repeat & 324 \\
trial-study-outcome & 307 \\
amazon-item-churn & 249 \\
stack-user-badge & 333 \\
hm-user-churn & 347 \\
\hline
\end{tabular}
\end{table}
\begin{table}[h!]
\centering
\caption{Evolution of Gaussian bias parameters during training.}
\label{tab:gaussian_stability}
\begin{tabular}{ccc}
\hline
\textbf{Epoch} & $\mu$ & $\sigma$ \\
\hline
1  & 33.31 & 9.98 \\
3  & 33.27 & 9.94 \\
5  & 33.24 & 9.90 \\
7  & 33.20 & 9.86 \\
9  & 33.16 & 9.83 \\
10 & 33.16 & 9.83 \\
\hline
\end{tabular}
\end{table}
From an efficiency perspective, increasing the sampling budget beyond 300 yields marginal performance gains but significantly increases computational and memory overhead due to quadratic attention complexity. In contrast, smaller budgets (e.g., below 200) tend to truncate informative relational signals, particularly in dense or heterogeneous graphs.

Overall, a budget of 300 achieves a favorable trade-off between semantic coverage, efficiency, and cross-dataset stability. This also aligns with our adaptive mechanism, where the learned effective neighborhood size naturally stabilizes around this range.

\subsection{Stability of Gaussian Bias Parameters}

We empirically study the optimization behavior of the Gaussian temporal bias parameters. The results show that these parameters exhibit fast and stable convergence during training, rather than instability.

From a theoretical perspective, as shown in Lemma 3.3, the gradient contains a linear restoring term that provides a global directional force towards a stable equilibrium. This design ensures well-conditioned optimization dynamics and prevents oscillatory updates.

We further track the parameter evolution on the trial-study-outcome dataset. As shown in Table~\ref{tab:gaussian_stability}, the parameters rapidly converge within the first few epochs and remain stable thereafter.

These results demonstrate that the Gaussian bias parameters are not sensitive to optimization instability and instead converge reliably to a 
stable region.

\newpage
\section{Related Work}

\paragraph{\textbf{Traditional Relational Learning}.}
Relational learning on relational databases was heavily reliant on manual feature engineering, requiring domain experts to transform multi-table records into flat feature vectors to accommodate tabular learning methods \citep{dwivedi2026relational}, e.g., LightGBM \citep{KeMFWCMYL17}, TabNet \citep{arik2021tabnet} and DNF-Net \citep{katzir2020net}. Specifically, TabNet treats attributes from distinct tables merely as a homogeneous feature set, while DNF-Net models interactions solely as property combinations within a single row. 
However, these methods will break the relationships between tables in the RDB. This reduction of hierarchical tables to flat vectors transforms \textit{Primary Key-Foreign Key} connections from explicit referential links to disconnected feature columns, effectively obscuring the hierarchical dependencies between tables.
Relational graph learning aims to address these challenges~\citep{abs-2002-02046, FeyHHLR0YYL24}.

\paragraph{\textbf{Relational Graph Learning (RGL)}.}
In RGL, the relational database is modeled as a relational graph, which effectively captures the relational dependencies. 
Most RGL methods are built upon the message-passing-based mechanisms, utilizing message-passing mechanisms to propagate information across tables. For example, RelGNN \citep{ChenKL25} and RDL~\citep{0001RHHHDFLYZHL24} adopt an iterative message-passing mechanism where row-level representations are updated by recursively aggregating transformed feature messages from connected neighbors across tables via foreign keys. 
However, these methods struggle to capture such long-range dependencies effectively \citep{WuJWMGS21, WuZLWY22}, limiting their ability to model complex and multi-hop relationships for downstream predictive tasks.
Therefore, existing methods shift to graph Transformer-based architectures, including RelGT~\citep{dwivedi2026relational}, Griffin~\citep{WangWGWYWZ25} and HGT \citep{HuDWS20}.
Nevertheless, current approaches fall short of addressing the complex structural, semantic, and temporal demands of RGL.
For example, RelGT relies on random sampling, which disregards explicit connections, leading to structural fragmentation and degrading predictive performance. 
Although Griffin \citep{WangWGWYWZ25} preserves all connected nodes, 
such full retention is often impractical in real-world databases, where a single row can be associated with thousands of records \citep{ChenKL25}. By indiscriminately incorporating such massive neighborhoods without semantic filtering, it inevitably introduces semantic noise \citep{WangJSWYCY19}.
In addition, they commonly flatten dynamic temporal patterns into static latent representations, ignoring the intrinsic temporal pattern and lacking the capacity to capture precise temporal dependencies.
Our GelGT explicitly addresses the above challenges, achieving state-of-the-art performance.

\section{Limitations and Broader Impacts}
\label{impact}
\subsection{Limitations}
\label{limitatiobs}
While GelGT achieves state-of-the-art predictive performance and strong computational efficiency on standard relational databases, its deployment in extreme ultra-high-frequency scenarios---such as millisecond-level financial trading---remains an open challenge. Such specialized environments impose extreme physical latency constraints that fall outside the typical scope of current relational graph learning frameworks, presenting an interesting direction for future engineering efforts.

\subsection{Broader Impacts}
\label{Broader Impacts}
Our work proposes a novel graph transformer architecture for relational graph learning, offering robust support for predictive tasks in relational databases and advancing the field of Machine Learning.
This work is foundational research in Machine Learning. We do not foresee any immediate negative societal impacts specific to this work.

\end{document}